\documentclass[12pt, oneside]{article}
\usepackage[a4paper, left=3cm, right=2cm, top=3cm, bottom=3cm]{geometry}

\usepackage[english]{babel}
\usepackage{amsmath,amsthm}
\usepackage{amsfonts}

\usepackage{latexsym}
\usepackage{tikz}
\usepackage{multicol}
\usepackage{amssymb}
\usepackage{mathrsfs}
\usepackage{oldgerm}
\usepackage{enumerate}
\usepackage{array}
\usepackage{xcolor}
\usepackage{stmaryrd}
\usepackage{hyperref}
\usepackage{bbm}
\usepackage{amstext}
\usepackage{stackrel}
\usepackage{mathdots}
\usepackage{mathtools}
\usepackage{cancel}
\usepackage{graphicx}
\usepackage{pdfpages}
\usepackage{float}      

\usepackage{enumitem}
\usepackage{bm}
\usepackage{esint}
\usepackage{mathrsfs}
\usepackage{extarrows}

\newcommand{\norm}[1]{\lVert#1\rVert}
\newcommand{\inner}[1]{\langle#1\rangle}

\newtheorem{theo}{Theorem}

\newtheorem{lem}[theo]{Lemma}
\newtheorem{remark}[theo]{Remark}

\title{High-dimensional classification problems with Barron regular boundaries under margin conditions}

\author{Jonathan Garc\'ia \footnote{Faculty of Mathematics, University of Vienna, Austria. \\ \texttt{jonathan.garcia.rebellon@univie.ac.at }} \and 
	Philipp Petersen  \footnote{Faculty of Mathematics and Research Network Data Science @ Uni Vienna, University of Vienna, Austria. \\ \texttt{philipp.petersen@univie.ac.at}}}
\allowdisplaybreaks

\begin{document}
	\emergencystretch 3em
	\maketitle
	\begin{abstract}
		We prove that a classifier with a Barron-regular decision boundary can be approximated with a rate of high polynomial degree by ReLU neural networks with three hidden layers when a margin condition is assumed. In particular, for strong margin conditions, high-dimensional discontinuous classifiers can be approximated with a rate that is typically only achievable when approximating a low-dimensional smooth function. We demonstrate how these expression rate bounds imply fast-rate learning bounds that are close to $n^{-1}$ where $n$ is the number of samples.
		In addition, we carry out comprehensive numerical experimentation on binary classification problems with various margins. We study three different dimensions, with the highest dimensional problem corresponding to images from the MNIST data set.
	\end{abstract}
	
	\medskip
	
	
	
	\medskip
	
	
\section{Introduction}

Binary classification is one of the oldest problems in machine learning, and it provides a foundation for understanding multiclass classification \cite{Bishop, Devroye, Trevor}. This problem can be reduced to interpreting whether an element $ \bm{x} $ is inside a set $ \Omega $ or not, and this can be represented mathematically by an indicator function $ \mathbbm{1}_{\Omega}(\bm{x}) $. Moreover, we define $ \Omega $ as 
\begin{equation}\label{posterandomeg}
	\Omega:=\left\{\bm{x}\in [0,1]^{d}: \eta(\bm{x})\geq 1/2\right\}~\text{ where }~ d\in\mathbb{N}_{\geq 2}, ~ \eta(\bm{x}):=\mu\left(Y=1|X=\bm{x}\right) 
\end{equation}
is the posterior probability, and denote the decision boundary of $ \Omega $  as $ \partial\Omega $. Our goal is to approximate $ \mathbbm{1}_{\Omega} $ using neural networks, which is a complex problem since  $ \mathbbm{1}_{\Omega} $ is a discontinuous function, can be very high dimensional, and it is known that approximation of high dimensional discontinuous functions is a serious problem in general \cite{Donoho}. It has also been shown that under classical smoothness assumptions on the decision boundaries $\partial\Omega$ the classification problem admits a curse of dimension \cite{PETERSEN2018296}.

Probably, the most common classification problem is image classification. 
In this case the dimension depends on the resolution of the image, i.e. the number of pixels used to create it. Typically this leads to a very high dimensional problem, some well known sample sets to train classification models are for example MNIST with images of $ 28\times28 $ pixels or ImageNet with images of $ 256\times256 $ pixels.

A classification method that uses a notion of distance to build the decision functions is expected to learn more easily in regions that are not close to the decision boundary. This suggest that learning should be relatively simple for distributions that do not have a lot of mass in the vicinity of the decision boundary. In order to control the mass in the vicinity of the decision boundary the following margin condition~\ref{(M)cond} is defined \cite{SVM}. Let $ \mu $ be a Borel probability measure on $ [0,1]^{d} $ and $ \mathcal{D} $ a joint distribution induced by $ \mu $ of an i.i.d.~sample set 
\begin{equation}
	S=\{(\bm{x}_{i},y_{i})\}_{i=1}^{n},\quad \text{where}\quad  (\bm{x}_{i},y_{i})\in [0,1]^{d}\times \{0,1\}\quad \text{and}\quad \mathbbm{1}_{\Omega}(\bm{x}_{i})=y_{i},\label{notsample}
\end{equation}
for all $ i\in \{1,\ldots,n\} $. Then, the margin condition is as follows.
\begin{enumerate}[label=(M)]
	\item\label{(M)cond} There exist $ C,\gamma>0 $, such that for all $ \epsilon>0 $, 
	\begin{equation*}
		\mu\left(B_{\epsilon}^{*}\right)\leq C\epsilon^{\gamma}\quad \text{where}\quad B_{\epsilon}^{*}:= \left\{\bm{x}\in[0,1]^{d}:\mathrm{dist}\left(\bm{x},\partial\Omega\right)\leq \epsilon\right\} 
	\end{equation*}
	is the ball of radius $ \epsilon>0 $ around $ \partial\Omega $ with respect to Euclidean distance $$ \mathrm{dist}\left(\bm{x},\partial\Omega\right):=\inf_{\bm{x'}\in \partial\Omega}\norm{\bm{x}-\bm{x'}}_{2}, $$ and $ \gamma $ is called the margin exponent.
\end{enumerate}
Furthermore, in \cite{NNbarronclass} it was also shown that by considering the class of measures $ \mu  $ satisfying the following condition, rates independent of the input dimension and appearing in the form of a polynomial factor can be obtained. It is said that $ \mu $ is tube compatible with parameters $ \alpha\in (0,1] $ and $ C>0 $ if for each measurable function $ f:[0,1]^{d-1}\to[0,1] $, each $ i\in \{1,\ldots,d\} $ and each $ \epsilon\in (0,1] $, it is satisfied that 
\begin{equation}
	\mu(T_{f,\epsilon}^{(i)})\leq C\epsilon^{\alpha}~ \text{ where } ~ T_{f,\epsilon}^{(i)}:=\left\{\bm{x}=(x_{1},\ldots,x_{d})\in [0,1]^{d}: |x_{i}-f(\bm{x}^{(i)})|\leq \epsilon\right\}\label{tube}
\end{equation}
and $\bm{x}^{(i)}:=(x_1,\ldots,x_{i-1},x_{i+1},\ldots,x_d)\in[0,1]^{d-1}$. The set $ T_{f,\epsilon}^{(i)} $ is called a tube of width $ \epsilon $ associated to $ f $.
Lastly, it was shown in \cite{YalewBarron,Barron1994,AppestB} that for a class of functions with finite Fourier moment, the curse of dimensionality can be overcome when the approximation error is measured in the $L^{2}$ norm. 
This influenced further research on the approximation in this class of functions (see e.g. \cite{CMa,Weinan}) which is nowadays known as the Barron class. 
Motivated by the qualities of this class of functions, we will use them throughout this paper. 
To be more precise, we study how complex the classification learning for deep neural networks (NNs) with the rectified linear unit (ReLU) activation function and using the hinge loss (see Section~\ref{NNsec}) can be, when the margin condition~\ref{(M)cond} and the tube compatibility \eqref{tube} are met, and $ \partial\Omega$ can be described by functions in Barron class (see Subsection~\ref{regulardb}). 
Under these conditions we obtained that:

\begin{enumerate}[label=(\roman*)]
	\item\label{results} Theorems \ref{firstaproxtheo} and \ref{secondproxtheo}.  It is possible to approximate $ \mathbbm{1}_{\Omega} $ using a ReLU-NN $ \Phi $ with three hidden layers and a number of neurons $ O(N) $ with $ N\in \mathbb{N} $, having an error of $ \mu\left(\left\{\mathbbm{1}_{\Omega}\neq \Phi\right\}\right)\lesssim N^{-\gamma/2} $	for all probability measures $ \mu $ satisfying the margin condition~\textnormal{\ref{(M)cond}} with margin exponent $ \gamma $, and the tube compatibility~\eqref{tube}. Moreover, when $ \Phi $ is trained using the hinge loss $ \phi $ (see \eqref{hinge01}), we obtain a learning rate (see Subsection~\ref{sectionlosses} and Remark~\ref{remark0}) of the order of $ O(n^{-\gamma/(2+\gamma)}(1+\log n)) $, where $ N\lesssim n^{\gamma/(2+\gamma)} $, $ n=|S| $ and $ S $ is a training set as in \eqref{notsample}. Implying that when $ \gamma $ is large, we achieve to approximate $ \mathbbm{1}_{\Omega} $ (high-dimensional discontinuous function) by $ \Phi $, at a rate comparable to that of a low-dimensional smooth function, with a fast learning rate of $ \approx n^{-1}(1+\log n) $. 
	
	\item Simulations~\ref{simult}. We produce numerical simulations within the framework of our main theorems: we generate samples in $ d\in\{3,50,784\} $ dimensions to simulate the margin condition~\ref{(M)cond} applied with a margin exponent $ \gamma $ and varying the number of elements $ n $ in various ranges. Then, we conclude as $\gamma$ and $ n$ increase, the test error (see Subsection \ref{sectionlosses}) decreases at a rate close to that stated in our Theorem~\ref{secondproxtheo} and Remark~\ref{remark0}, which supports our results.
	
\end{enumerate}

\subsection{Background}

In \cite[Section 8]{SVM}, the features of SVMs when applied to binary classification were investigated. It was shown that for distributions with noise but low concentration (margin condition~\ref{(M)cond}) near the decision boundary, the approximation error for Gaussian kernels was relatively small, leading to favorable learning rates. A data-dependent strategy for selecting the regularization and kernel parameters was also studied, introducing the noise exponent to measure the amount of high noise in the labeling process. The learning rates achieved were sometimes as fast as $ n^{-1} $, where $ n $ is the number of data points.

Next in \cite{fastc}, it was shown that a convergence rate similar to the rates of \cite{SVM} can be achieved but using NN classifiers. Particularly, it was proved that the estimated classifier based on NNs with the hinge loss achieves similar fast convergence rates considering the following assumptions on the true classifier:
\begin{enumerate}[label=(\Roman*)]
	\item\label{noise} The noise Tsybakov condition is always met \cite{noisec}. That is, there exist $ C>0 $ and $ q\in[0,\infty] $ such that for any $ \epsilon>0 $,
	\[
	\mu\left(\left\{\bm{x}\in[0,1]^{d}:|2\eta(\bm{x})-1|<\epsilon\right\}\right)\leq C\epsilon^{q}.
	\]
	\item\label{holder} The decision boundary is H\"older smooth, i.e. $ \partial\Omega$ can be described by functions in the space
		\[
	\mathcal{H}^{\ell}(\mathcal{X})=\left\{f\in \mathcal{C}^{[\ell]^{-}}(\mathcal{X}): \max_{\norm{\bm{m}}\leq[\ell]^{-}} \norm{\partial^{\bm{m}}f}_{\infty} 
	+\max_{\norm{\bm{m}}=[\ell]^{-}} \left[\partial^{\bm{m}}f\right]_{\{\ell\}^{+}}
	<\infty  \right\},
	\]
	where $ \ell>0 $, $[\ell]^{-}=\left\lceil \ell -1 \right\rceil$, $\{\ell\}^{+}=\ell-[\ell]^{-} $, $ \mathcal{C}^{m}(\mathcal{X}) $ is the space of $ m\in \mathbb{N} $ times differentiable functions on a set $ \mathcal{X} $ whose partial derivatives of order $ \bm{m} $ with $ \norm{\bm{m}}\leq m $ are continuous;
	\[
	\partial^{\bm{m}}f=\frac{\partial^{\norm{\bm{m}}}f(\bm{x})}{\partial \bm{x}^{\bm{m}}}\quad \text{and}\quad [f]_{s}=\sup_{\underset{\bm{x}\neq \bm{y}}{\bm{x},\bm{y}\in\mathcal{X}}}\frac{|f(\bm{x})-f(\bm{y})|}{\norm{\bm{x}-\bm{y}}^{s}}.
	\]	
	\item The case when the conditional class probability is assumed to be H\"older smooth, meaning that $ \eta(\cdot)\in \mathcal{H}^{\ell}(\mathcal{X}) $, where $ \eta $ is as in \eqref{posterandomeg}.
	
	\item\label{maarginkimhom} When the margin condition~\ref{(M)cond} is fulfilled, but with margin exponent $\gamma\in[1,\infty] $.
\end{enumerate}
However, when the margin condition is satisfied, the bound shown in \cite{fastc} for the learning rate is
	\begin{equation}
		\left(\frac{\log^{3}n}{n}\right)^{\frac{\ell(q+1)}{\ell(q+2)+(d-1)(q+1)/\gamma}},\text{ for which }~\lim_{\underset{q,\ell,\gamma \text{ are fixed} }{d\to \infty}} \frac{\ell(q+1)}{\ell(q+2)+(d-1)(q+1)/\gamma} =0 \label{boundkim}
	\end{equation}
and the variables $ d,q,\ell,\gamma $ are defined in items \ref{noise}, \ref{holder}, \ref{maarginkimhom} above, meaning that in some sense bound \eqref{boundkim} is affected by the curse of dimensionality and to avoid it, one must require that $ d\lesssim \gamma $, but this is a rather restrictive condition. 

So, our results \ref{results} provide a significant improvement over the bound \eqref{boundkim} when $ \ref{(M)cond} $ is satisfied, since they are not affected by the curse of dimensionality and show faster convergence.

\subsection{Neural networks}\label{NNsec}
Neural networks are functions formed by connecting neurons, where the output of one neuron becomes the input to another. Here, a neuron is a function of the form 
\[
\mathbb{R}^{d}\ni \bm{x}\to \sigma(\left\langle\bm{w},\bm{x}\right\rangle+b)
\]
where $ \bm{w}\in\mathbb{R}^{d} $ is a weight vector, $ b\in\mathbb{R} $ is called bias, and the function $ \sigma $ is referred to as an activation function. There are several types of NNs but in this paper we only work with one of the most common ones, which is the so-called feedforward NN, in this structure, neurons are organized in layers, and the neurons in each layer receive input only from the previous layer. We introduce these NNs as in \cite{Pbook}, as follows.

Let $ L\in\mathbb{N} $, $ d_0 ,\ldots,d_{L+1}\in \mathbb{N}$ and  $\sigma:\mathbb{R}\to \mathbb{R}$ an activation function. Then, we call a function $ \bm{\Phi}:\mathbb{R}^{d_{0}}\to \mathbb{R}^{d_{L+1}} $ a neural network if there exist for $ \ell\in\{0,\ldots,L\} $, matrices $ \bm{W}^{(\ell)}\in \mathbb{R}^{d_{\ell+1}\times d_{\ell}}  $ and vectors $ \bm{b}^{(\ell)}\in \mathbb{R}^{d_{\ell+1}} $ such that for all $ \bm{x}\in \mathbb{R}^{d_{0}} $,
\[
\bm{x_{0}}:=\bm{x},\qquad \bm{x}^{(\ell)}:=\sigma(\bm{W}^{(\ell-1)}\bm{x}^{(\ell-1)}+\bm{b}^{(\ell-1)}) \quad\text{for all}\quad\ell\in\{1,\ldots,L\},
\]
and
\[
\bm{\Phi}(\bm{x})=\bm{x}^{(L+1)}:=\bm{W}^{(L)}\bm{x}^{(L)}+\bm{b}^{(L)}.
\]
In addition, we define the following parameters
	\begin{flalign*}		
		L(\bm{\Phi}):=L+1  &\qquad \text{number of layers,}&&\\
		N(\bm{\Phi}):=\sum_{j=0}^{L+1}d_{\ell}& \qquad \text{number of neurons,}&&\\
		\Theta(\bm{\Phi}):=\left(\left(\bm{W}^{(0)},\bm{b}^{(0)}\right),\ldots,\left(\bm{W}^{(L)},\bm{b}^{(L)}\right)\right)&\qquad \text{parameter tuple,}&&\\
		\norm{\bm{\cdot}}_{0}&\qquad \text{number of non-zero entries of } \bm{\cdot},&&\\
		W(\bm{\Phi}):=\sum_{\ell=0}^{L}\left(\norm{\bm{W}^{(\ell)}}_{0}+\norm{\bm{b}^{(\ell)}}_{0}\right)& \qquad \text{number of weights,}&&\\
		W_{\infty}(\bm{\Phi}):=\underset{0\leq \ell\leq L}{\max}\left\{\underset{0\leq \ell\leq L}{\max}\norm{\bm{W}^{(\ell)}}_{\infty},\underset{0\leq \ell\leq L}{\max}\norm{\bm{b}^{(\ell)}}_{\infty}\right\}& \qquad \text{largest absolute value of }\,\Theta, &&\\	
		d_{0}\text{ and }d_{L+1}  &\qquad \text{input and output dimensions,}&&\\
		\left(d_0,d_1,\ldots,d_{L+1}\right)\in\mathbb{N}^{L+2} &\qquad \text{architecture,}&&\\
		L  &\qquad \text{number of hidden layers.}&&
	\end{flalign*}
Moreover, when $ L=d_{2}=1 $, we call $ \Phi $ a shallow neural network, where
\begin{equation}
	\Phi(\bm{x})=b^{(1)}+\sum_{i=1}^{d_1}w_{i}^{(1)}\sigma\left(\left\langle \bm{w}_{i}^{(0)},\bm{x} \right\rangle+b_{i}^{(0)}\right)\label{eqSNN}
\end{equation}
for some $ w_{i}^{(1)}, b_{i}^{(0)}, b^{(1)}\in\mathbb{R} $ and $ \bm{w}_{i}^{(0)}\in\mathbb{R}^{d_{0}} $ for $ i=1,\ldots,d_{1} $.

In the sequel, we use the the ReLU activation function which we denote as $ \varrho:\mathbb{R}\to \mathbb{R} $ and which is defined as $ \varrho(x):=\max\{0,x\} $.
Let $d\in \mathbb{N}_{\geq 2}, N,W\in \mathbb{N} $ and $ B>0 $. We denote by $ \mathcal{NN}=\mathcal{NN}(d,N,W,B) $ the set of ReLU-NNs with three hidden layers, $ d $-dimensional input and 1-dimensional output, with at most $ N $ neurons per layer, with at most $ W $ non-zero weights, and with these weights bounded in absolute value by $ B $. 
Furthermore, we define
\begin{equation}
	\mathcal{NN}_{*}=\mathcal{NN}_{*}(d,N,W,B)
	:=\left\{f\in \mathcal{NN}: 0\leq f(\bm{x})\leq 1\quad \text{for all} \quad \bm{x}\in [0,1]^{d}\right\}.\label{notatNN}
\end{equation}

\subsubsection{Loss function and the hinge loss}\label{sectionlosses}

In supervised learning, determining the optimal parameters of a NN is achieved by minimizing an objective function. This objective is defined based on a collection of input-output pairs known as a sample. Concretely, let $ S=\{(\bm{x}_{i},y_{i})\}_{i=1}^{n} $ be as in \eqref{notsample}, the goal is to find a NN $ \bm{\Phi} $ such that $ \bm{\Phi}(\bm{x}_{i})\approx y_{i} $  for all $ i\in \{1,\ldots,n\} $ in a meaningful sense. For this purpose, a loss function $ \mathcal{L} $ is defined which measures the dissimilarity between its inputs, and the so-called empirical risk of $ \bm{\Phi} $ with respect to the sample $ S $ and $ \mathcal{L} $, defined as 
\begin{equation}
	\widehat{\mathcal{E}}_{\mathcal{L},S}(\bm{\Phi})=\frac{1}{n}\sum_{i=1}^{n}\mathcal{L}(\bm{\Phi}(\bm{x}_{i}),y_{i})\label{empirisk}
\end{equation}
is minimized. Next, after optimization, we want to measure the performance of  $ \bm{\Phi} $ on new data $ (\bm{x}_{\text{new}},y_{\text{new}}) $ sampled independently from $ \mathcal{D} $, for this we define the so-called risk (or $ \mathcal{L} $-risk) of $ \bm{\Phi} $, in the form 
\[
\mathcal{E}_{\mathcal{L}}(\bm{\Phi})=\mathbb{E}_{(\bm{x}_{\text{new}},y_{\text{new}})\sim \mathcal{D}}\left[\mathcal{L}(\bm{\Phi}(\bm{x}_{\text{new}}),y_{\text{new}})\right]:=\mathbb{E}_{\bm{x}}\left[\mathcal{L}(\bm{\Phi}(\bm{x}),y)\right].
\]
If the risk is not much larger than the empirical risk, then we say that the NN $ \bm{\Phi} $ has a small generalization error. Otherwise, if the risk is much larger than the empirical risk, then we say that $ \bm{\Phi} $ overfits the training data, meaning that it has memorized the training samples, but fails to generalize to new data (see \cite{fastc,Pbook}).

In this context, let $ d\in\mathbb{N} $; $ \mu $ be a Borel probability measure on $ [0,1]^{d} $; $ \mathcal{D} $ a joint distribution induced by $ \mu $ of the set $ S $, where $ \mathcal{D}_{X} $ and $ \mathcal{D}_{y} $ are the respective marginal distributions of $ X:=\{\bm{x}_{i}\}_{i=1}^{n} $ and $ y:=\{y_{i}\}_{i=1}^{n} $; 
\begin{equation}
	f\in \mathcal{F}:=\left\{f:[0,1]^{d}\to [0,1] : f \text{ is a }\mu\text{-measurable function} \right\};\label{defF}
\end{equation}
the classifier $ C_{f}(\bm{x}):= \mathbbm{1}_{[1/2,1]}(f(\bm{x})) $ and label $ y\in \{0,1\} $.  We define  
\begin{equation}
	\mathcal{L}^{*}(C_{f}(\bm{x}),y):=\mathbbm{1}_{(2C_{f}(\bm{x})-1)(2y-1)<0} \quad\text{and}\quad \mathcal{E}:=\mathcal{E}_{\mathcal{L}^{*}},\label{notatloss0}
\end{equation}
as the  $ 0-1 $ loss and the $ 0-1 $ risk, respectively; and we want to obtain the minimizer
\begin{equation}
	C^{*}:=\underset{f\in\mathcal{F}}{\operatorname{argmin}}\,\mathcal{E}(C_{f})\label{notatloss}
\end{equation} 
called the Bayes classifier which is the function that has the most accurate classification on average from a sample set. The classifier $ C^{*} $ could be approximated by the empirical risk minimization for the same loss function $ \mathcal{L}^{*} $, that is 
\[
C^{*}\approx\underset{f\in\mathcal{F}}{\operatorname{argmin}}\,\widehat{\mathcal{E}}_{\mathcal{L}^{*},S}(C_{f}),
\]
but finding the minimizer of $\widehat{\mathcal{E}}_{\mathcal{L}^{*},S}$ is NP hard (see \cite{Bartlett}). 
However, in \cite{Bartlett,Zhang}, it was proven that if the loss function is the hinge loss, denoted by $ \phi $ and defined\footnote{With a slight difference, since in this case $ \phi $ is defined on $ [0,1]\times\{0,1\} $ instead of $ \mathbb{R}\times\{\pm 1\} $.} as 
\begin{align}
	\phi:[0,1]\times\{0,1\}&\to [0,\infty)\nonumber\\
	(x,y)& \mapsto \max\{0,1-(2y-1)(2x-1)\},\label{hinge01}
\end{align}
then 
\[ 
C^{*}=\underset{f\in\mathcal{F}}{\operatorname{argmin}}\,\mathcal{E}_{\phi}(C_{f}),
\] 
and this is one of the main reasons why from now on we continue to work with the hinge loss in the remainder of this paper. 
It is important to consider that the fast convergence rate of the excess $ \mathcal{L} $-risk $ \mathcal{E}_{\mathcal{L} }(f,C^{*}):=\mathcal{E}_{\mathcal{L} }(f)-\mathcal{E}_{\mathcal{L} }(C^{*}) $ does not always imply the fast convergence rate of the excess $ 0-1 $ risk $ \mathcal{E}(f,C^{*}):= \mathcal{E}(f)-\mathcal{E}(C^{*}) $.

But for the hinge loss $ \phi $, some constant $ C_{\phi}>0 $ (which depends on $ \phi $) and any $ f\in \mathcal{F} $, the inequality
\begin{equation}
	\mathcal{E}(f,C^{*})\leq C_{\phi} \mathcal{E}_{\phi}(f,C^{*})\label{ineqforhinge}
\end{equation}
holds (see \cite[inequality (2.3)]{fastc}). This is the second reason why we choose to work with hinge loss, and inequality~\eqref{ineqforhinge} will be used to prove our main results.

\subsubsection{Covering entropy}
We need to introduce the notion of covering entropy given in \cite[Definition 3.9]{petersen2021optimal} from \cite[Definitions 1 and 2]{YangBarron}.

Let $ \mathcal{X}\neq \emptyset $ be a set and let $ {\rm dist}:\mathcal{X}\times \mathcal{X}\to [0,\infty] $ be a distance function. For a set $ \emptyset\neq K\subset \mathcal{X} $ and $ \epsilon>0 $, we call a set $ G_{\epsilon}\subset \mathcal{X} $ an $ \epsilon $-net for $ K\subset \mathcal{X}$ if  
\begin{equation}\label{defentropy}
	\text{ for all}\quad x\in K \quad\text{there exists} \quad y\in G_{\epsilon} \quad\text{satisfying}\quad {\rm dist}(x,y)\leq \epsilon.
\end{equation}
Then, we define 
\begin{align*}
	\mathcal{G}_{K,{\rm dist}}(\epsilon)&:= \min\left\{|G|: G\subset \mathcal{X} \text{ is an } \epsilon\text{-net for } K\right\}~\text{ and }\\
	V_{K,{\rm dist}}(\epsilon)&:=\ln(\mathcal{G}_{K,{\rm dist}}(\epsilon)),
\end{align*}
where $ V_{K,{\rm dist}}(\epsilon) $ is called the $ \epsilon $-covering entropy of $ K $.

The next result provides an upper bound for the covering entropy of $ \mathcal{NN}_{*} $ and is a consequence of \cite[Lemma 6.1]{GrohsFelix} presented in \cite[Remark 5.2]{petersen2021optimal}.

\begin{lem}\label{lemaentropy}
	Fixing $ \mathcal{X}:=\mathcal{NN}_{*}(d,N,W,B) $, $ K:=[0,1]^{d} $ and $ \operatorname{dist}:=\norm{\cdot}_{\infty} $, we have
	\begin{equation*}
		V_{[0,1]^{d},\norm{\cdot}_{\infty}}(\delta)
		\leq W\cdot\left(10+\ln(1/\delta)+5\ln(\lceil B \rceil)+5\ln(\max\{d,W\})\right)
	\end{equation*}
	for $\delta\in (0,1], d,N,W\in \mathbb{N} $ and $ B>0 $.
\end{lem}

\subsection{Fourier-analytic Barron space}

There are several slightly different interpretations of the term ``Barron function'' in the literature, see for example \cite[Section 7]{NNbarronclass}. We use this term to refer to the Fourier-analytic notion of Barron-type functions in \cite{NNbarronclass,Barron1994, petersen2021optimal}. Consequently, we adopt the definition of \cite[Definition 2.1]{NNbarronclass} with parameters $ X:=[0,1]^{d} $ and $ x_{0}:=\bm{0} $, i.e., as follows.

A function $ f:[0,1]^{d}\to \mathbb{R} $ is said to be of Barron class with constant $ C>0 $, if there are $ c\in [-C,C] $ and a measurable function $ F:\mathbb{R}^{d}\to \mathbb{C} $ satisfying
\begin{equation}
	\int_{\mathbb{R}^{d}}|F(\bm{\xi})|\sup_{\bm{x}\in [0,1]^{d}}\left|\left\langle \bm{\xi}, \bm{x}\right\rangle\right|d\bm{\xi} \leq C\quad\text{and}\quad f(\bm{x})=c+\int_{\mathbb{R}^{d}} (e^{i\inner{\bm{x},\bm{\xi}}}-1)\cdot F(\bm{\xi})d\bm{\xi}\label{barroncondition0}
\end{equation}
for all $ \bm{x}\in[0,1]^{d} $. The set of all these Barron functions is denoted as  $ \mathcal{B}_{C} $.

In \cite{Barron1994} it was shown that for a certain class of functions with bounded variation, shallow NNs of $ N $ neurons achieve an approximation accuracy of order $ N^{-1/2} $ in the $ L^{2} $ norm. Subsequently, in \cite[Proposition 2.2]{NNbarronclass} and \cite[Theorem 12]{WECMaSWLWu}, the same accuracy was demonstrated for approximating functions of the Barron class using shallow ReLU neural networks with $ N $ neurons under the supremum norm. This allows us to formulate \cite[Definition 3.1]{NNbarronclass}. 

For $C>0$, we define the {Barron approximation set} $\mathcal{BA}_{C} $ as the set of all functions $f:[0,1]^{d}\to\mathbb{R}$  such that for all $N\in\mathbb{N}$, there is a shallow ReLU neural network (shallow ReLU-NN) $ {\Phi} $ with $N$ neurons in the hidden layer such that 
\begin{equation}
	{\norm{f-{\Phi}}}_{\infty}\leq C\sqrt{d/N}\label{defbarronaproxset}
\end{equation}
and all weights and biases of ${\Phi}$ are bounded in absolute value by 
\[
\sqrt{C}\left(5+\vartheta\right),\quad \text{where}\quad \vartheta:=\sup_{\bm{\xi}\in\mathbb{R}^{d}\setminus\{0\}}\left(\frac{\norm{\bm{\xi}}_{\ell^{\infty}}}{\sup_{\bm{x}\in [0,1]^{d}}\left|\left\langle \bm{\xi}, \bm{x}\right\rangle\right|}\right).
\]
In addition, we know that $ \sqrt{C}\left(5+\vartheta\right)\leq 7\sqrt{C} $, since for $\widetilde{\bm{x}}=\frac{1}{2}\left(\frac{\bm{\xi}}{\norm{\bm{\xi}}}+\bm{1}\right)\in [0,1]^{d}$,
\[
\sup_{\bm{x}\in [0,1]^{d}}\left|\left\langle \bm{\xi}, \bm{x}\right\rangle\right|\geq \left|\left\langle \bm{\xi},\widetilde{\bm{x}}\right\rangle\right|=\frac{\left|\left\langle \bm{\xi}, \bm{\xi}\right\rangle\right|}{2\norm{\bm{\xi}}}+\frac{\left|\left\langle \bm{\xi}, \bm{1}\right\rangle\right|}{2}\geq \frac{\norm{\bm{\xi}}}{2}
\]
and $ \vartheta\leq\sup_{\bm{\xi}\in\mathbb{R}^{d}\setminus\{0\}}\left({2\norm{\bm{\xi}}_{\ell^{\infty}}}/{\norm{\bm{\xi}}}\right)\leq 2 $.

The set $\mathcal{BA}=\bigcup_{C>0}\mathcal{BA}_{C}$ is called Barron approximation space. In fact, $ \mathcal{B}_{C/\kappa_{0}}\subset \mathcal{BA}_{C} $ for every $ C>0 $ and a constant $ \kappa_{0}>0 $ that is absolute (i.e. independent of all other quantities and objects) (see \cite[Remark 3.2]{NNbarronclass}).

\subsubsection{Barron regular boundaries}\label{regulardb}

We now define the Barron regular boundaries using \cite[Definitions 2.2 and 2.3]{petersen2021optimal} which are compatible with \cite[Definition 3.3]{NNbarronclass} for $ \mathcal{BA}_{C} $ by \cite[Remark 2.4]{petersen2021optimal}. Let $ M,d\in \mathbb{N} $ with $ d\geq 2 $, let a set $ \emptyset\neq \mathscr{C}\subset C([0,1]^{d-1};[0,1]) $ and $ b\in \mathscr{C} $. We define the general horizon function associated with $ b $ as 
\begin{align*}
	h_{b}:[0,1]^{d}&\to \{0,1\},\\
	x=(x_{1},\ldots,x_{d})&\mapsto \mathbbm{1}_{b(x_{1},\ldots,x_{d-1})\leq x_{d}},
\end{align*}
and the set of general horizon functions associated to $ \mathscr{C} $ as $ H_{\mathscr{C}}:=\{h_{b}:b\in \mathscr{C}\}. $

Moreover, we say that a compact set $ \Omega \subset[0,1]^{d} $ has $ (\mathscr{C},M) $-regular decision boundary if there exist (closed, axis-aligned, non-degenerate) rectangles $ Q_{1},\ldots,Q_{M}\subset [0,1]^{d} $ such that $ \Omega\subset \bigcup_{i=1}^{M}Q_{i} $, where the $ Q_{i} $ have disjoint interiors and such that either 
\begin{equation}\label{coveras}
	\mathbbm{1}_{\Omega}=g_{i}\circ P_{i}\quad\text{or}\quad \mathbbm{1}_{\Omega}=1-g_{i}\circ P_{i}\quad \text{almost everywhere on}\quad Q_{i},
\end{equation}
for a general horizon function $ g_{i}\in H_{\mathscr{C}} $ associated to $ \mathscr{C} $ and a $ d $-dimensional permutation matrix $ P_{i} $. Also, a family $ \{Q_{i}\}_{i=1}^{M} $ of rectangles as above is called an associated cover of $ \Omega $. We write $ \mathcal{R}_{\mathscr{C}}(d,M) $ for all sets with $ (\mathscr{C},M) $-regular decision boundary. Then, we define the set of all classifiers with $ (\mathscr{C},M) $-regular decision boundary as 
\[
\mathcal{C\ell}_\mathscr{C}(d,M):=\left\{\mathbbm{1}_{\Omega}:\Omega\in \mathcal{R}_{\mathscr{C}}(d,M)\right\}.
\]

Finally, Barron horizon functions and classifiers with Barron-regular decision boundary are the elements of $ H_{\mathcal{BA}_{C}} $ and $ \mathcal{C\ell}_{\mathcal{BA}_{C}}(d,M) $.

\section{Main results}

Our main results confirm the intuitive idea that using the margin condition with an appropriate margin exponent for sets with Barron-regular decision boundary can drastically improve the way NNs learn. To establish our results, we consider the following parameters:

\begin{enumerate}[label=(P\arabic*)]
	\item\label{parametrship} Let $d\in\mathbb{N}_{\geq 2};M,N\in\mathbb{N};C_{1},C_{2}, C_{3}>0$; $ \gamma>0 $, $ \alpha\in(0,1] $ and $\Omega\in \mathcal{R}_{\mathcal{BA}_{C_{1}}}(d,M)$, where $ d $ is the dimension of $ \Omega $; $ M $ is the number of sets in  the associated cover of $ \Omega $ (see \eqref{coveras}); $ C_{1} $ is the constant in the right hand side of \eqref{defbarronaproxset}; $ C_{2} $ and $ \gamma $ are the margin condition parameters in \ref{(M)cond}; and finally, $ C_{3} $ and $ \alpha $ are the tube compatibility parameters in \eqref{tube}.  
	
	\item\label{parametrship2} With parameters~\ref{parametrship}, let
	\begin{align*}
		\widehat{N}_{n}&:=\left\lceil(7Md)^{2/\gamma}(d-1)C_{1}^{2}\max\{C_{2},C_{3}\}^{2/\gamma} n^{2/(2+\gamma)} \right\rceil,\\
		N_{n}&:= \left\lceil M(4(d+1)+\widehat{N}_{n}+1)+d+1\right\rceil,\\
		W_{n}&:= \left\lceil 41Md^{2}\widehat{N}_{n}\right\rceil,\\
		B_{n}&:= \left\lceil(1+\sqrt{C_{1}})\left(7+N/C_{1}+(N/C_{1})^{\gamma/\alpha}\right)\right\rceil,\qquad \text{for all}\quad n\in \mathbb{N},
	\end{align*}
	where $ n=|S| $ and $ S $ is a training set as in \eqref{notsample}.
	
\end{enumerate}
Our first result shows that classifiers in $ \mathcal{C\ell}_{\mathcal{BA}_{C_{1}}}(d,M) $ are well approximated by ReLU-NNs when the margin condition is assumed.

\begin{theo}\label{firstaproxtheo}
	Let the parameters be as defined in \textnormal{\ref{parametrship}}. There exists a ReLU-NN $\Phi$ with three hidden layers such that, for all probability measures $ \mu $ satisfying the margin condition~\textnormal{\ref{(M)cond}}, and the tube compatibility~\eqref{tube}, it holds that
	\[
	\mu\left(\left\{\bm{x}\in [0,1]^{d}:\mathbbm{1}_{\Omega}(\bm{x})\neq \Phi(\bm{x})\right\}\right)\leq 3.5Md(d-1)^{\gamma/2}C_{1}^{\gamma}N^{-\gamma/2}\max\{C_{2},C_{3}\}.
	\]
	Moreover, $0\leq \Phi(\bm{x})\leq 1$ for all $\bm{x}\in [0,1]^{d}$ and the architecture of $\Phi$ is given by 
	\[
	\mathcal{A}=\left(d,M(2(d+1)+N),M(2d+2),M,1\right).
	\]
	Thus, $\Phi$ has at most $M(4(d+1)+N+1)+d+1$ neurons and at most $ 41Md^{2}N $ non-zero weights. The weights and biases of $ \Phi $ are bounded in magnitude by 
	\[
	(1+\sqrt{C_{1}})\left(7+N/C_{1}+(N/C_{1})^{\gamma/\alpha}\right).
	\]
\end{theo}
To prove this result, we use the property defined in \eqref{coveras} for $ \Omega $ with Barron-regular decision boundary, which allows us to identify $ \Omega $ in each of the sets associated to its covering through Barron horizon functions. Then, in each set of the $ \Omega $ covering we use the approximation \eqref{defbarronaproxset} for shallow ReLU-NNs, and construct the three hidden layers using the ReLU function. Next, we sum the NNs created in each set of the covering and obtain the final NN for the whole $ \Omega $ set. Lastly, we find the architecture of this final NN and bound in absolute value its weights and biases.

The next step is to measure the performance of this approximation using the $ 0-1 $ loss. Therefore, we present the following result.

\begin{theo}\label{secondproxtheo}
	Let the parameters be as defined in \textnormal{\ref{parametrship}} and \textnormal{\ref{parametrship2}}. Then, for all probability measures $ \mu $ satisfying the margin condition~\textnormal{\ref{(M)cond}}, and the tube compatibility~\eqref{tube}, we have that 	
	\begin{equation}\label{conclusiontaprox2}
		\mu\left(\mathbb{E}_{\bm{x}}\left[\mathbbm{1}_{[1/2,1]}\left(\widehat{f}_{\phi,S}(\bm{x})\right)\neq \mathbbm{1}_{\Omega}(\bm{x})\right]\gtrsim n^{-\gamma/(2+\gamma)}(1+\log n)  \right)\lesssim n^{-1},
	\end{equation}
	where 
	\[
	\widehat{f}_{\phi,S}:=\underset{f\in\mathcal{F}}{\operatorname{argmin}}\,\widehat{\mathcal{E}}_{\phi,S}(f)\in \mathcal{NN}_{*}(d,N_{n},W_{n},B_{n}),
	\] 
	$ \widehat{\mathcal{E}}_{\phi,S}(f) $ is the $ \phi $-empirical risk as in \eqref{empirisk},  $ \mathcal{F} $ is defined in \eqref{defF} and $ \mathcal{NN}_{*} $ in \eqref{notatNN}.
\end{theo}
In order to prove Theorem~\ref{secondproxtheo}, we use \cite[Theorem A.1]{fastc} with a slight adaptation to our hypothesis and obtain Lemma~\ref{theoA1mod} as an auxiliary result.

\begin{remark}\label{remark0}
	Theorem~\ref{secondproxtheo} implies
	\[
	\mathbb{E}_{S,\bm{x}}\left[\mathbbm{1}_{[1/2,1]}\left(\widehat{f}_{\phi,S}(\bm{x})\right)\neq \mathbbm{1}_{\Omega}(\bm{x})\right]\lesssim n^{-\gamma/(2+\gamma)}(1+\log n).
	\]
	Indeed, we denote for a moment  $ \Gamma_{S,\bm{x}}:=\left\{\mathbbm{1}_{[1/2,1]}\left(\widehat{f}_{\phi,S}(\bm{x})\right)\neq \mathbbm{1}_{\Omega}(\bm{x})\right\} $ and observe that \eqref{conclusiontaprox2} implies 
	\begin{align*}
		\mathbb{E}_{S,\bm{x}}\left[\Gamma_{S,\bm{x}}\right]&=\mathbb{E}_{S}\left[\mathbb{E}_{\bm{x}}\left[\Gamma_{S,\bm{x}}\right]\right]\\
		&\lesssim \mathbb{E}_{S}\left[\mathbb{E}_{\bm{x}}\left[\Gamma_{S,\bm{x}}\right]\mathbbm{1}_{\mathbb{E}_{\bm{x}}\left[\Gamma_{S,\bm{x}}\right]\gtrsim n^{-\gamma/(2+\gamma)}(1+\log n) }\right] + n^{-\gamma/(2+\gamma)}(1+\log n)\\
		&\leq  \mu\left(\mathbb{E}_{\bm{x}}\left[\Gamma_{S,\bm{x}}\right]\gtrsim n^{-\gamma/(2+\gamma)}(1+\log n)\right) +n^{-\gamma/(2+\gamma)}(1+\log n)\\
		&\lesssim n^{-1}  +n^{-\gamma/(2+\gamma)}(1+\log n)\\
		&\lesssim n^{-\gamma/(2+\gamma)}(1+\log n),
	\end{align*}
	here we use the fact that $ S $ was employed to define the probability measure $ \mu $ and therefore $ \mathbb{E}_{S}\left[X\right]=\int_{\Omega}Xd\mu $ holds for every random variable $ X $.\qed

\end{remark}
To close, we must emphasize that when $ \gamma $ is large we achieve in Theorem \ref{firstaproxtheo}, the approximation of a high-dimensional discontinuous function at a rate comparable to that of a low-dimensional smooth function. Moreover, we get in Theorem~\ref{secondproxtheo} and Remark~\ref{remark0}, a fast learning rate of $ \approx n^{-1}(1+\log n) $ if $ \gamma $ is large.

\subsection{Numerical experiments}\label{simult}

In this section, we present some numerical simulations that support the results of Theorems \ref{firstaproxtheo} and \ref{secondproxtheo}, in practical scenarios.  We build classifiers using TensorFlow in Python with different training samples and architectures that fulfill the hypotheses of our results. For a clear explanation of our algorithm, we detail each step below.

\subsubsection{Sample generation} We denote our samples as $ S_{d,n}^{\gamma} $ where  $ \gamma $ is the margin exponent in \ref{(M)cond},  $ d\in\{3,50,784\} $ is the dimension of $ \Omega:=\Omega^{d}\subset[0,1]^{d} $,  and $ n=|S_{d,n}^{\gamma}| $ is the number of examples. 
 
  Subsequently, we define the characteristics of our samples and the distance $ \mathrm{dist}\left(\bm{x},\partial\Omega^{d}\right) $ in order to apply the margin condition~\ref{(M)cond}.\\

\noindent\textbullet~ For $ d\in \{3,50\} $, we construct a data set where we choose the hypersurface $ \partial\Omega^{d} $ that separates the labels as the positive part of the hypersphere of dimension $d$ and radius $1/2$. Concretely, let  
\begin{align*}
	\mathcal{X}_{r}&:=\left\{\bm{x}\in [0,1]^{d}: \bm{x}\text{ has positive entries and } \norm{\bm{x}}\leq r\right\}~~ \text{and}\\\widetilde{\mathcal{X}}_{r}&:=\left\{\bm{x}/r:\bm{x} \in \mathcal{X}_{r}\right\}.
\end{align*}
To avoid problems with the volume of hypespheres in high dimensions for $ r<1 $, we start by taking $ r=4 $ and generating $ 3\times10^{5} $ examples in $ \mathcal{X}_{4} $ for training and $ 10^{6} $ for testing, making sure that $ |\mathcal{X}_{4}\setminus\mathcal{X}_{2}|=|\mathcal{X}_{2}| $.  However, our examples must be in $ [0,1]^{d} $, so we define them such that $ S_{d,n}^{\gamma}\subset \widetilde{\mathcal{X}}_{4} $ (see Figure~\ref{d_3}) and the distance between any $ \bm{x}\in S_{d,n}^{\gamma} $ and $ \partial\Omega^{d} $ as $ \mathrm{dist}\left(\bm{x},\partial\Omega^{d}\right):=|1/2-\norm{\bm{x}}| $.

\begin{figure}[H]
	\centering
	\scalebox{0.46}
	{\includegraphics{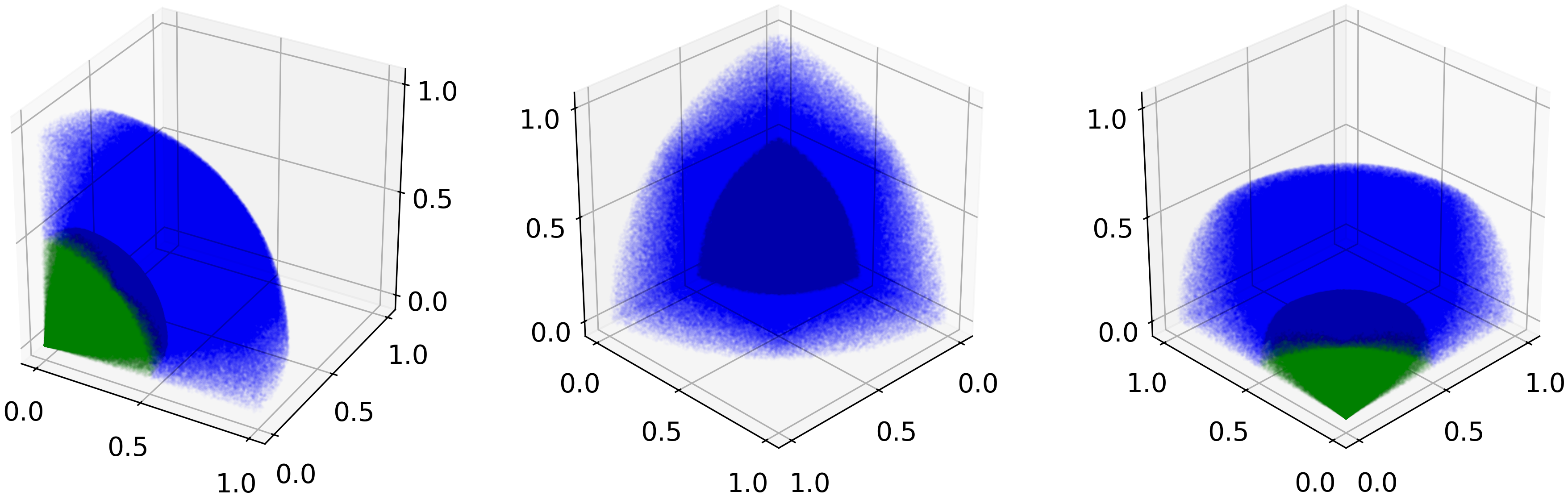}}
	\caption{Examples test for $ d=3 $.}
	\label{d_3}
\end{figure}

\noindent\textbullet~ For $ d=784 $, we focus on an example based on real data. Therefore, we use the MNIST data set to have images of size $ 28\times28 $ pixels, and we classify the images of numbers $ 0 $ and $ 1 $. Since MNIST does not have a large number of elements and we want several samples to observe the effect of the margin condition~\ref{(M)cond} and the number of elements, we use the SMOTE \cite{smote} technique (similar to the method used in \cite{fastc} to generate artificial samples) to equalize the number of elements in each label and increase the sample of each label to $95\%$ more than the larger original, obtaining in total $ 18576 $ examples for training and $ 12146 $ for test (see Figure~\ref{smotef}).
\begin{figure}[H]
	\centering
	\scalebox{0.363}
	{\includegraphics{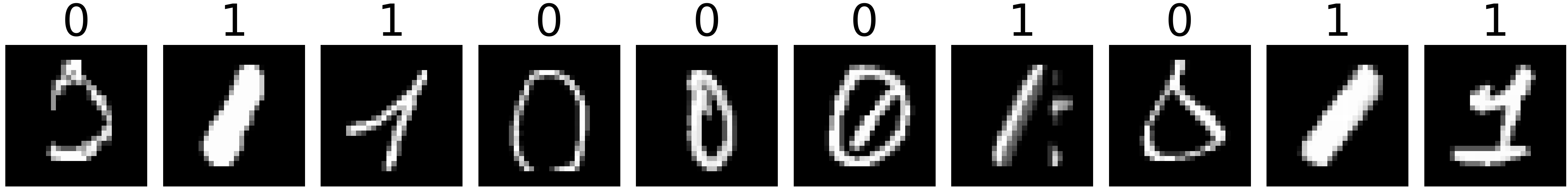}}
	\caption{Some examples after smote for $ d=784 $.}
	\label{smotef}
\end{figure}

In this sense, we define $ S_{784,n}^{\gamma} $ such that its elements are vectors in $ [0,1]^{784} $ constructed from matrices in $ \mathbb{R}^{28\times28} $ flattened and with entries normalized to $ [0,1] $ representing images of $ 0 $ and $ 1 $. However, to apply the margin condition~\ref{(M)cond} we must have an idea of the distance between the examples and the decision boundary, so we use an initial NN\footnote{$ \Phi_{0} $ is a ReLU-NN with architecture $ (784,256,128,64,1) $, which was obtained using hinge loss and sigmoid function is applied to the output.} $ \Phi_{0} $ that assigns probabilities to each example and form a vector $ \bm{w}=(w_{i})_{i=1}^{n} $ with these predictions (and keeping only the well predicted examples), i.e.~$ w_{i}=\Phi_{0}(\bm{x}) $ for some $ \bm{x}\in S_{784,n}^{\gamma} $ and $ i\in\{1,\ldots,n\} $, where $ \mathbbm{1}_{[1/2,1]}\left(\Phi_{0}(\bm{x})\right) $ corresponds to the original label of $ \bm{x} $. Then, $ \bm{w} $ gives us an idea of the position of the examples and the possible decision boundary, using their probabilities in the interval $ [0,1] $.  Thus, we assume for a moment as the decision boundary, the set of $ \bm{x}\in [0,1]^{784} $ with $ \Phi_{0}(\bm{x}) \in [\ell,u] $ for some $ \ell,u\in (0,1) $ such that $ \left[\min_{{i\in\{1,\ldots,n\}}}\{w_{i}\},\ell\right)\cup\left(u,\max_{{i\in\{1,\ldots,n\}}}\{w_{i}\}\right]  $ contains approximately $ 99.9 \% $  of the labels $ \{w_{i}\}_{i=1}^{n} $ (see Figure~\ref{dboundry784}).
\begin{figure}[H]
	\centering
	\scalebox{0.32}
	{\includegraphics{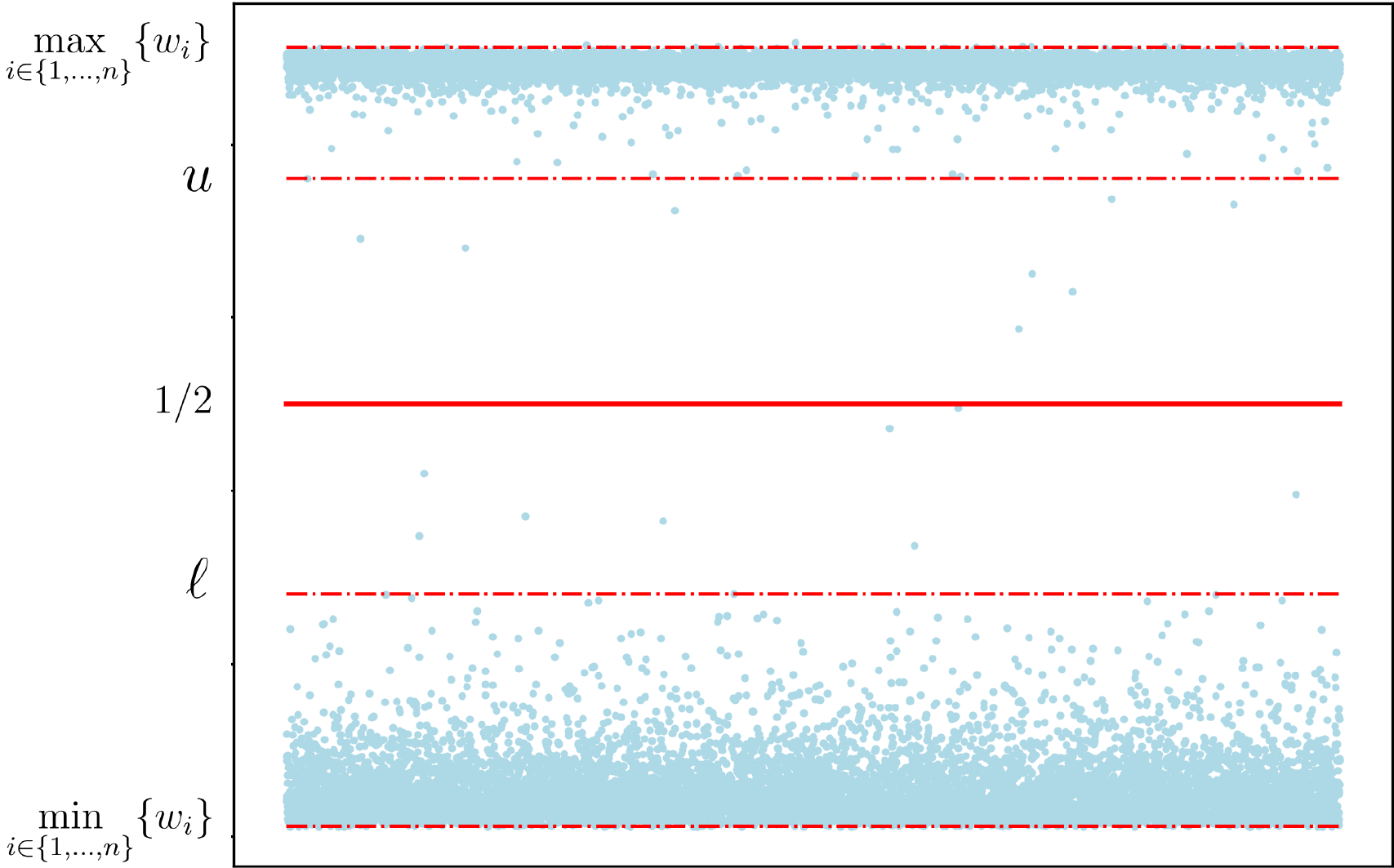}}
	\caption{Train sample probability applying $ \Phi_{0} $ and decision boundary for $ d=784 $.}
	\label{dboundry784}
\end{figure}

Next, we remove the $ 0.1\% $ of elements with label in $  [\ell,u]  $ and rescale the values of $ \bm{w} $ with the mapping $ t:[0,1]^{784}\to [0,1] $ that maps $ \min_{{i\in\{1,\ldots,n\}}}\{w_{i}\}$ to $0 $;
$ \ell,u$ to $1/2 $; and $ \max_{{i\in\{1,\ldots,n\}}}\{w_{i}\}$ to $1 $. 
Finally, after this process, we obtain $ 18337 $ elements for the training set (see Figure~\ref{Final_train_784}) and $ 12053 $ for the test set, with the same number of elements for each label.
To simulate a margin condition, we require, for a given point $ \bm{x}\in S_{784,n}^{\gamma} $, its distance to the decision boundary  $ \partial\Omega^{784} $. This is computationally very hard to identify. 
Therefore, we use the expression $|1/2-t(\Phi_{0}(\bm{x}))|$ as a proxy for the distance, since
\begin{equation}
	\mathrm{dist}\left(\bm{x},\partial\Omega^{784}\right) \gtrsim |1/2-t(\Phi_{0}(\bm{x}))| \quad\text{for all}\quad i\in \{1,\ldots,n\}.\label{disMNIST}
\end{equation}
This proxy is reasonable because of the following considerations:
\begin{enumerate}[label=$\ast$]
	\item $ \partial\Omega^{784}=\partial\{\bm{x}\in[0,1]^{d}:\eta(x)\approx t(\Phi_{0}(\bm{x})) \geq 1/2\} $.
	\item  $ t(\Phi_{0}(\bm{x})) $  is a modified sigmoid function of the initial sigmoid  $ \Phi_{0}(\bm{x}) $.
	\item  $t(\Phi_{0}(\bm{x}))=1/(1+e^{-\left(\left\langle\widehat{\bm{w}},\bm{x}\right\rangle+b\right)}) $ for some weight vector $ \widehat{\bm{w}} $ and bias $ b $, after they have been modified by the use of the $ t $ function.
	\item $ G:=\left\{\bm{x}\in[0,1]^{d}:  \left\langle\widehat{\bm{w}},\bm{x}\right\rangle+b=0 \right\} $ is the projection of the decision boundary in the characteristics space, therefore we assume
	\[
	\mathrm{dist}\left(\bm{x},\partial\Omega^{784}\right)\approx \mathrm{dist}\left(\bm{x},G\right)= |\left\langle\widehat{\bm{w}},\bm{x}\right\rangle+b|/\norm{\widehat{\bm{w}}}.
	\]
	\item Approximating $ 1/(1+e^{-y}) $ around $ y=0 $, implies that $ 1/(1+e^{-y})\approx 1/2 + y/4 $. So, when $ y=\left\langle\widehat{\bm{w}},\bm{x}\right\rangle+b $, we obtain 
	\[
	|1/2-t(\Phi_{0}(\bm{x}))|\approx |\left\langle\widehat{\bm{w}},\bm{x}\right\rangle+b|/4\lesssim \mathrm{dist}\left(\bm{x},\partial\Omega^{784}\right).
	\]
	In conclusion, to apply the margin condition~\ref{(M)cond} in this case, we use $ |1/2-t(\Phi_{0}(\bm{x}))| $ instead of the exactly $ \mathrm{dist}\left(\bm{x},\partial\Omega^{784}\right) $, which is a stronger assumption.
\end{enumerate}
\begin{figure}[H]
	\centering
	\scalebox{0.7}
	{\includegraphics{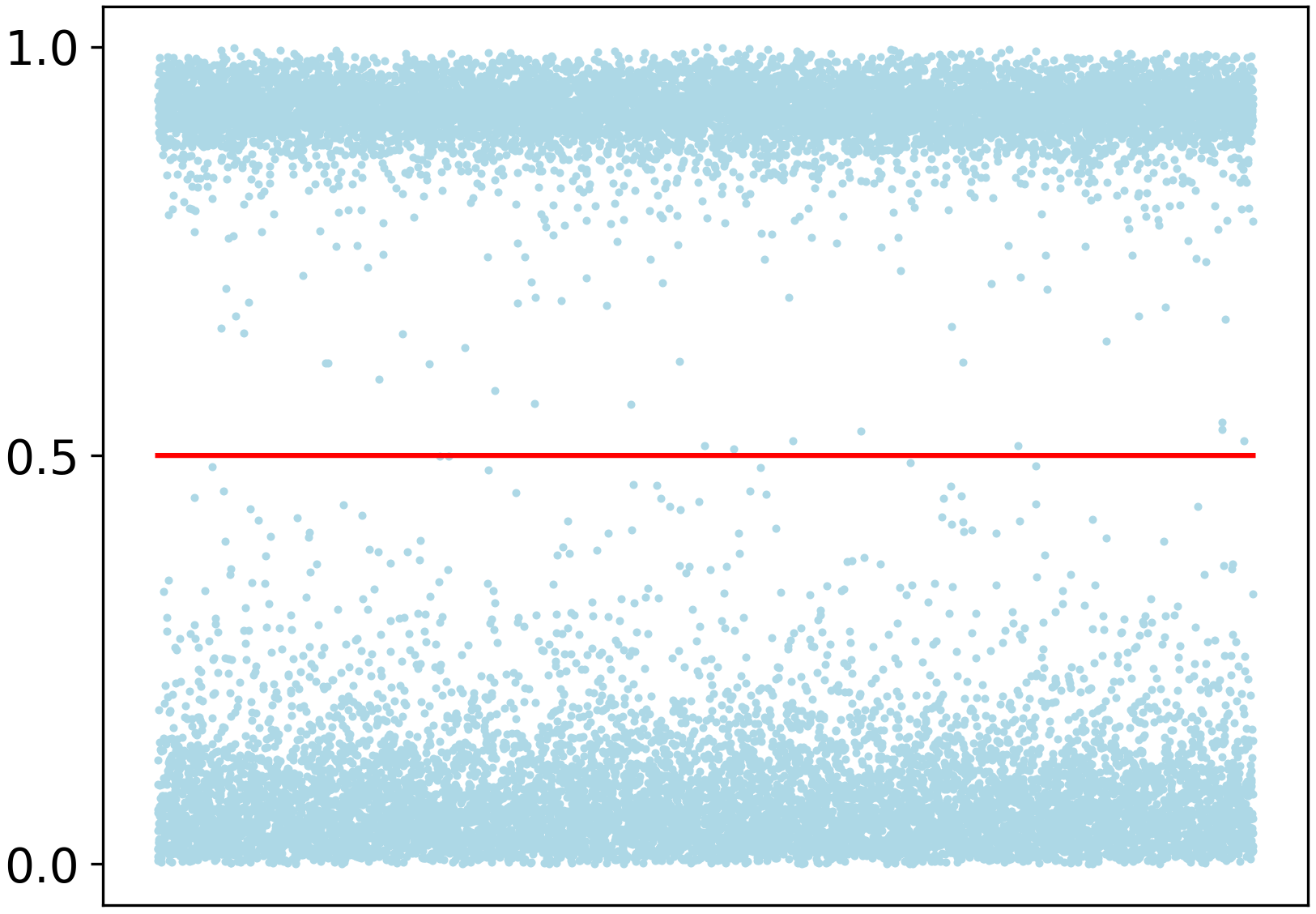}}
	\caption{Final train sample after SMOTE and location of the possible $ \partial\Omega^{784} $.}
	\label{Final_train_784}
\end{figure}

To the basis sets defined in the previous items to describe the characteristics of the elements of $ S_{d,n}^{\gamma} $, we apply the margin condition~\ref{(M)cond} with margin exponent 
\[
\gamma\in \Gamma:=\{0.1, 0.644, 1.189, 1.733, 2.278, 2.822, 3.367, 3.911, 4.456, 5.0\},
\]
defining a neighborhood centered at $ 0.5 $ of radius $ c_{d}=0.48 $ when $ d\in\{3,50\} $ or $ c_{d}\approx 0.5-2.8\cdot10^{-7} $ for $ d=784 $, and assigning to each point a weight of $ 0 $ if it is outside this neighborhood or $ 1-(\mathrm{dist}\left(\bm{x},\partial\Omega^{d}\right)/c_{d})^{\gamma}  $ if it is not. So, we randomly decide to remove from the sample with probability according to the weight.  Let
\begin{small}
		\begin{align*}
		G_{1}:=\,&\{499,	730,	1065,	1556,	2271,	3317,	4843, 7071,	10323,	15073,	22007, 32130,	46911,	68492\},\\
		G_{2}:=\,&\{249,	321,	414,	533,	687,	885,	1140,	1468,	1890,	2435,	3135,	4038,	5200, 6696,	8624\}.
	\end{align*}
\end{small}
Then, for each $ \gamma $ we randomly select subsets of the samples, such that 
\begin{align*}
	n\in G(d):=\begin{cases}
		G_{1}\cup \{120001\}&\text{ if } d=3\\
		G_{1}\cup \{120088\}&\text{ if } d=50\\
		G_{2}&\text{ if } d=748
	\end{cases}\qquad \text{and}\qquad 
	 n= n_{d}:=\begin{cases}
		399644&\text{ if } d=3\\
		399187&\text{ if } d=50\\
		4837&\text{ if } d=748,
	\end{cases}
\end{align*}
for train and test, respectively  (see e.g. Figures \ref{margin-9} and \ref{d3margin}) .

\begin{figure}[H]
	\centering
	\scalebox{0.285}
	{\includegraphics{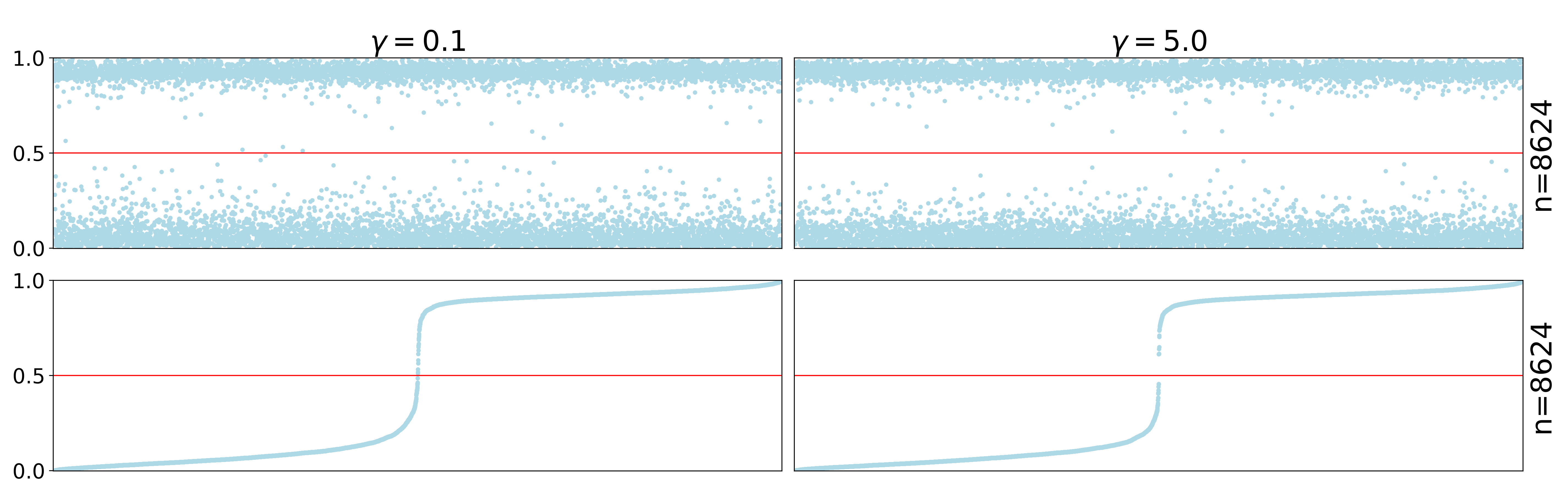}}
	\caption{Example of margin condition applied for $ d=784 $.
		The second row contains the points of the first row but in ascending order.}
	\label{margin-9}
\end{figure}
\begin{figure}[H]
	\centering
	\scalebox{0.68}
	{\includegraphics{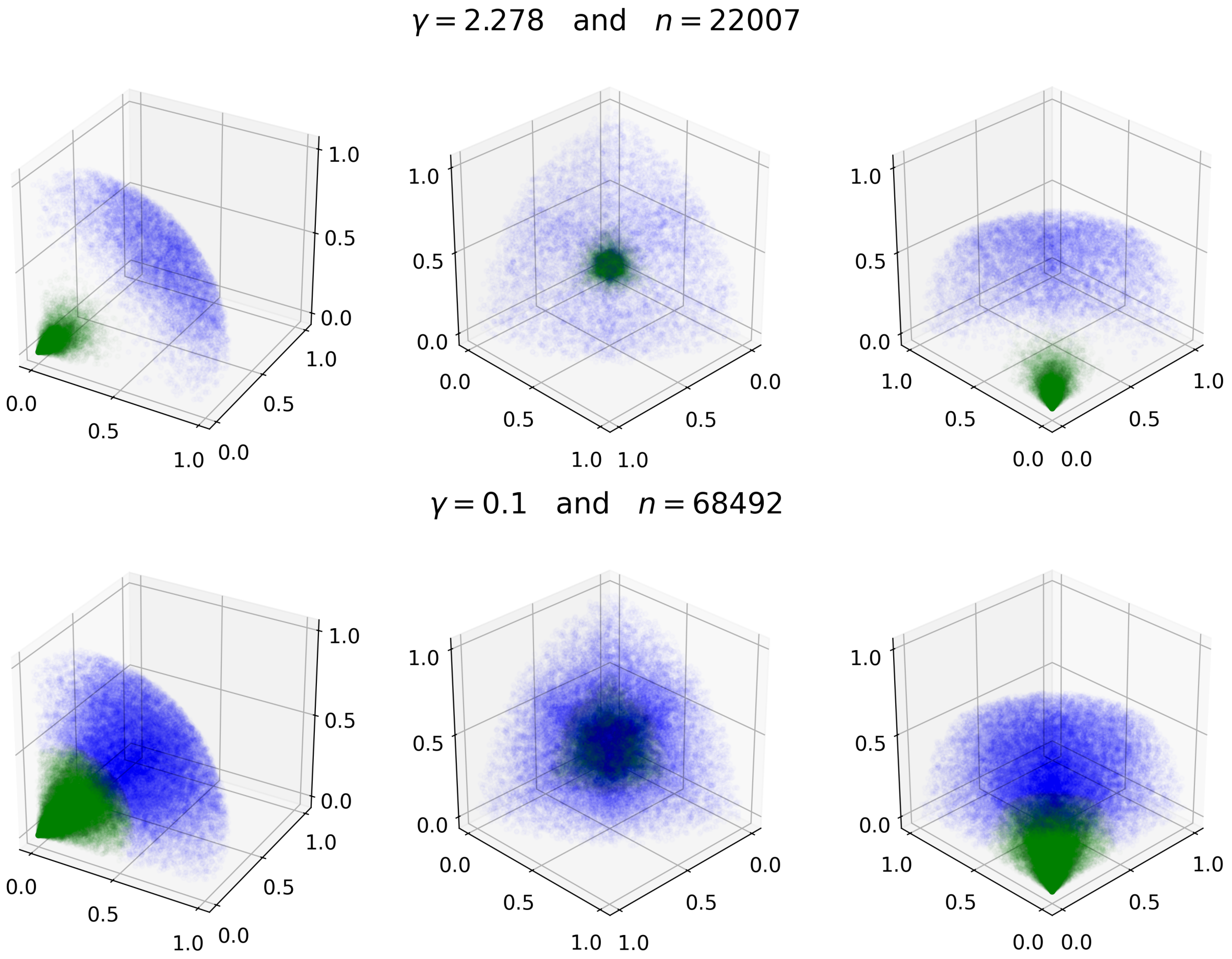}}
	\caption{Example of margin condition applied for $ d=3 $.}
	\label{d3margin}
\end{figure}

Lastly, we repeat the process up to $ 61, 57, 200  $ iterations for $d=3,50,784$,  respectively, and obtain different samples to train and test, which will then be used to have an average result.

\subsubsection{NN training}

With the $ S_{d,n}^{\gamma,\text{train}}\in \left\{S_{d,n}^{\gamma}\right\}_{n\in G(d)} $ samples previously generated as training sets, we train ReLU-NNs with architecture
\[
(d,3N,2N,N,1) \quad \text{where}\quad N=\lceil n^{2/(\gamma+2)}\rceil,
\]
using the hinge loss $ \phi $ and applying the sigmoid function to the output, we also use the optimizer Adam with a learning rate of $ 1e-5 $, we apply early stop monitoring loss with patience equal to 1 for 50000 epochs, and we do not use validation sets as it is usually done, since our main generalization result does not describe this situation, but the errors minimizerof the training risk. 
Next, we compute the empirical $ 0-1 $ risk  as in \eqref{empirisk} and \eqref{notatloss0}, using the samples  $  S_{d,n}^{\gamma,\text{test}}\in \left\{S_{d,n}^{\gamma}\right\}_{n=n_{d}} $ previously defined as test sets, and obtain  an approximation of
\[
\mathbb{E}_{S_{d,n}^{\gamma,\text{train}},S_{d,n}^{\gamma,\text{test}}}\left[\mathbbm{1}_{[1/2,1]}\circ\widehat{f}_{\phi,S_{d,n}^{\gamma,\text{train}}}\neq \mathbbm{1}_{\Omega}\right] 
\]
as in the Remark~\ref{remark0}.  Moreover, we computed the empirical $\phi$-risk to make the interpretation\footnote{This interpretation is valid to support our results because it is known that $ \mathcal{E}(f)\lesssim \mathcal{E}_{\phi}(f) $ for any $f\in \mathcal{F}$, when $ C^{*}:=\mathbbm{1}_{\Omega} $ is chosen. (see \eqref{ineqforhinge} and \eqref{equalsphi01risk}).} of our results easier, since empirical  $0-1$ risk is slightly less stable. For each $ d\in \{3,50,784\} $ and $ \gamma\in \Gamma $, our results can be seen in Figures \ref{784} and \ref{784_01}, for the empirical $\phi$-risk and  $0-1 $ risk, respectively.

\begin{figure}[H]
	\centering
	\scalebox{0.35}
	{\includegraphics{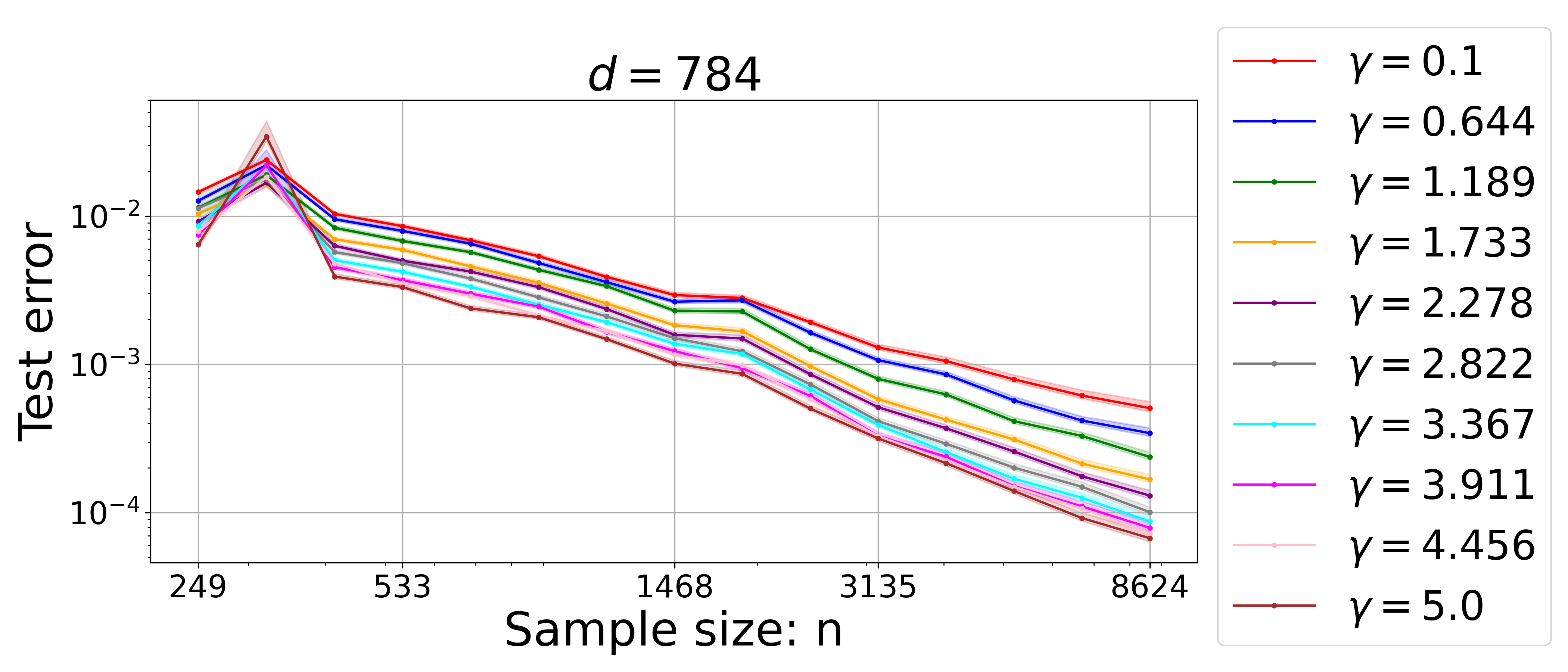}}
	\scalebox{0.55}
	{\includegraphics{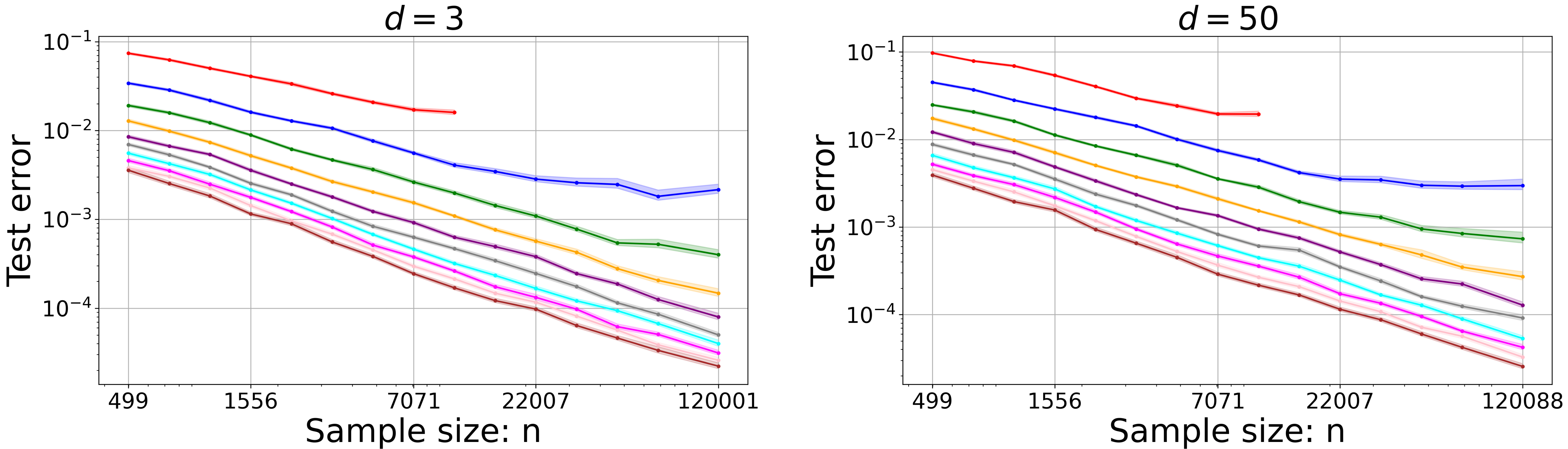}}
	\caption{Plot $ n $ vs empirical  $ \phi $-risk.  We averaged over 61,57,200 iterations in the 3,50,784 cases, respectively. The error shadow shows one relative standard deviation.}
	\label{784}
\end{figure}

\noindent\textbullet~ The $ (d,\gamma)\in \{3,50\}\times\{0.1\} $ cases could only be computed as far as seen in Figure~\ref{784} because the memory requirements of the data and architecture became excessive.\\

\noindent\textbullet~  We repeated the experiments obtaining different $ S_{d,n}^{\gamma,\text{train}}$ and $S_{d,n}^{\gamma,\text{test}} $ in several iterations, then applied mean on all these results and plotted an error shadow based on the standard deviation of the results in these iterations. 

In conclusion the simulations support our results (Theorem \ref{secondproxtheo} and Remark \ref{remark0}). Furthermore, Figure~\ref{784} shows that the rate approaches $n^{-1}$ for large margin and does not for small margin. The $\gamma = 5.0$ curve spans close to two orders of magnitude.  The number of samples spans $1.5$ orders of magnitude, so the error rate is close to or even slightly exceeding the rate $n^{-1}$. 
The case $\gamma = 0.1$, yields a reduction of the loss by roughly one order of magnitude, so the associated error rate is worse than $n^{-1}$.  For the other examples, the $ x $ axis goes through 2.4 orders of magnitude. The curves with the highest margin also manage a bit more than 2 orders of magnitude. The ones with lower margin (green for example) only manage 1, so closer to the $n^{-1/2}$ rate.

\begin{figure}[H]
	\centering
	\scalebox{0.33}
	{\includegraphics{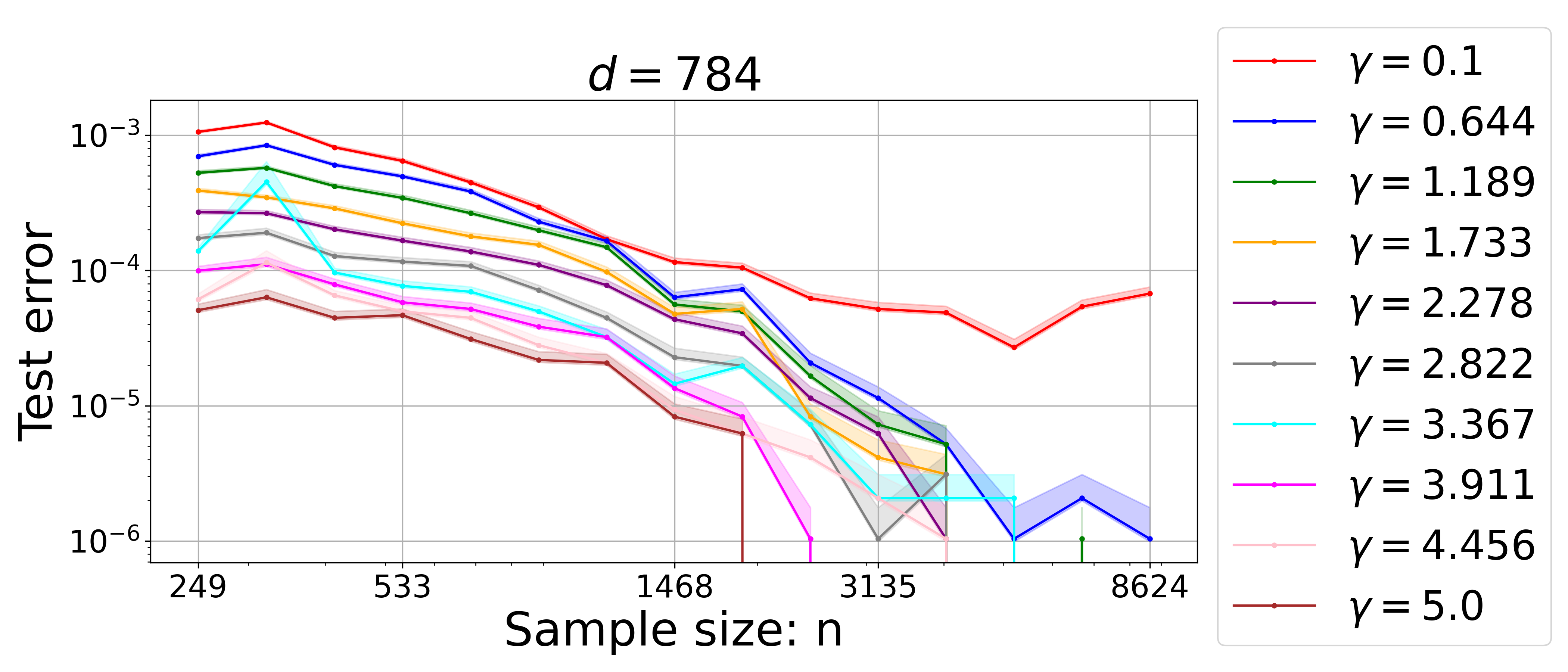}}
	\scalebox{0.53}
	{\includegraphics{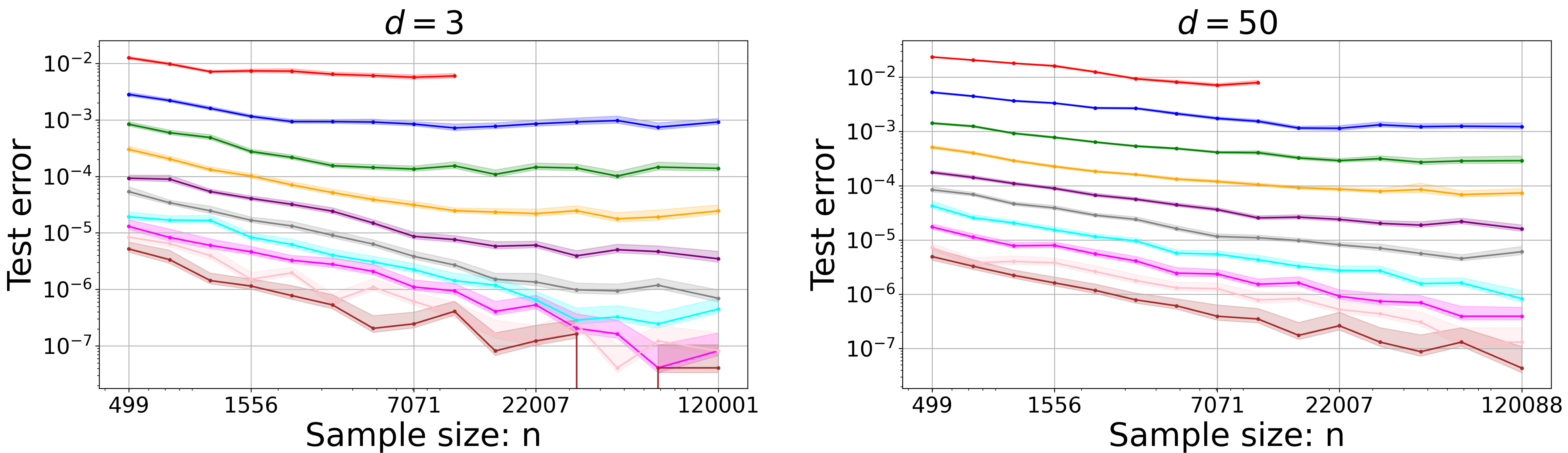}}
	\caption{Plot $ n $ vs empirical $ 0-1 $ risk.  We averaged over 61,57,200 iterations in the 3,50,784 cases, respectively. The error shadow shows one relative standard deviation.}
	\label{784_01}
\end{figure}

Figure~\ref{784_01} supports the interpretation made for Figure~\ref{784} and, as expected, the empirical 0-1 risks are even lower than the corresponding  empirical $\phi$-risks. 

\section*{Acknowledgements}

J.G. and P.C.P. were supported by the Austrian Science Fund (FWF) Project P-37010.


\bibliography{margincond}

\appendix

\section{Appendix}

\subsection{Proof of Theorem~\ref{firstaproxtheo}}
Since $ \Omega\in \mathcal{R}_{\mathcal{BA}_{C_1}}(d,M) $, there exists an associated cover of $ \Omega $, called  $ \{Q_{i}\}_{i=1}^{M} $, such that property \eqref{coveras} is satisfied for each $ i\in\{1,\ldots,M\} $, with $ P_{i}(\bm{x})=(x_{i1},\ldots,x_{id}) $ where $ x_{i1},\ldots,x_{id} $ are the entries of vector $ \bm{x} $ after being permuted via $ P_i $, and $ g_{i}(P_{i}(\bm{x}))=\mathbbm{1}_{b_{i}(x_{i1},\ldots,x_{i(d-1)})\leq x_{id}} $ for some $ b_{i}\in \mathcal{BA}_{C_1}\subset C([0,1]^{d-1};[0,1]) $. Note that property~\eqref{barroncondition0}  is fulfilled even for $ f_{i}\in \mathcal{BA}_{C_1} $ defined as
\[
f_{i}(\bm{x}^{(i)}):=f_{i}(x_{1},\ldots,x_{i-1},x_{i+1},\ldots,x_{d})=b_{i}(x_{i1},\ldots,x_{i(d-1)})  \quad\text{and}\quad  x_{i}:=x_{id},
\]
given that it has been shown in \cite[A.2]{petersen2021optimal} that in this case $ f_{i} $ can be seen as a composition of $ b_{i} $ with permutations of the vector $ \bm{x}^{(i)} $. Therefore, fixing $ m\in \{1,\ldots,M\} $, we have
\[
\mathbbm{1}_{Q_{m}\cap\Omega}(\bm{x})=\mathbbm{1}_{x_{i}\leq f_{m}(\bm{x}^{(i)})}\quad \text{or}\quad \mathbbm{1}_{Q_{m}\cap\Omega}(\bm{x})=\mathbbm{1}_{x_{i}\geq f_{m}(\bm{x}^{(i)})}\quad \text{a.e.  on}\quad Q_{m}.
\]
Without loss of generality (otherwise the proof is analogous) we say 
	\[
	\mathbbm{1}_{Q_{m}\cap\Omega}(\bm{x})=\mathbbm{1}_{x_{i}\leq f_{m}(\bm{x}^{(i)})}~\text{ for }~Q_{m}:=\prod_{i=1}^{d}[a_{i},b_{i}]\subseteq [0,1]^{d}~\text{ and }~\bm{x}^{(i)}\in Q_{m}\setminus\prod_{\underset{j\neq i}{j=0}}^{d}[a_{j},b_{j}].
	\] 
By inequality \eqref{defbarronaproxset} and as $ f_{i}\in \mathcal{BA}_{C_1}  $, there exists a shallow ReLU-NN $ {\Phi}_{m} $ such that 
\begin{equation}
	{\norm{f_{m}-{\Phi}_{m}}}_{\infty}\leq \delta:= C_{1}\sqrt{(d-1)/N}\label{ineqRSNNT2}
\end{equation}
and all weights and biases of ${\Phi}_{m}$ are bounded in absolute value by $ 7\sqrt{C_{1}} $.

Let $ G_{m}:=\left\{\bm{x}\in Q_{m} : \mathbbm{1}_{Q_{m}\cap\Omega}(\bm{x})\neq \mathbbm{1}_{x_{i}\leq {\Phi}_{m}(\bm{x}^{(i)})} \right\} $ and observe that 
\begin{align*}
	G_{m}=\,&\left\{\bm{x}\in Q_{m} : \mathbbm{1}_{x_{i}\leq f_{m}(\bm{x}^{(i)})}\neq \mathbbm{1}_{x_{i}\leq {\Phi}_{m}(\bm{x}^{(i)})} \right\}\\
	=\,&\left\{\bm{x}\in Q_{m} : x_{i}-f_{m}(\bm{x}^{(i)})\leq 0\right\}\triangle \left\{\bm{x}\in Q_{m} : x_{i}-f_{m}(\bm{x}^{(i)})\leq {\Phi}_{m}(\bm{x}^{(i)})-f_{m}(\bm{x}^{(i)})\right\}\\
	\subseteq\,& \left\{\bm{x}\in Q_{m} : |x_{i}-f_{m}(\bm{x}^{(i)})|\leq |{\Phi}_{m}(\bm{x}^{(i)})-f_{m}(\bm{x}^{(i)})|\leq \delta\right\},
\end{align*}
where $ \triangle $ denotes the symmetric difference of sets. If $ \bm{x}=\left(x_{1},\ldots,x_{d}\right)\in G_{m} $ and $ \bm{z}^{*}=\left(x_{1},\ldots,x_{i-1},f_{m}(\bm{x}^{(i)}),x_{i+1},\ldots,x_{d}\right)$ is in the decision boundary  $ \partial\Omega $, we get 
\begin{equation}
	\mathrm{dist}({\bm{x},\partial\Omega})=\inf_{\bm{x}^{*}\in \partial\Omega}\norm{\bm{x}-\bm{x}^{*}}_{2}\leq\norm{\bm{x}-\bm{z}^{*}}_{2} = |x_{i}-f_{m}(\bm{x}^{(i)})|\leq \delta,\label{eqdist}
\end{equation}
therefore $ \bm{x}\in B_{\delta}^{*} $; this implies $ G_{m} \subseteq B_{\delta}^{*} $. As the margin condition~\ref{(M)cond} is satisfied with $ \gamma>0 $, we have 
\begin{equation}
	\mu(G_{m})\leq \mu\left(B_{\delta}^{*}\right)\leq C_{2}\delta^{\gamma}\label{mG}.
\end{equation}
Now, we define the ReLU-NN
\begin{align*}
	&\widetilde{\Phi}_{m}(\bm{x}):=H_{\delta}\left(\Phi_{m}(\bm{x}^{(i)})-x_{i}\right), &\text{where}\quad H_{\delta}(x)&:=\begin{cases}
		0 & \text{if } x\leq 0  \\
		x/\delta & \text{if } 0\leq x\leq \delta \\
		1 & \text{if } x\geq \delta  
	\end{cases}\\&&&\,= \frac{1}{\delta}\left(\varrho(x)-\varrho(x-\delta)\right).
\end{align*}
We also note that if 
\begin{align*}
	\bm{x}\in\widetilde{G}_{m}:=\left\{\bm{x}\in Q_{m} : \mathbbm{1}_{x_{i}\leq {\Phi}_{m}(\bm{x}^{(i)})} \neq \widetilde{\Phi}_{m}(\bm{x}) \right\}	
	=\left\{\bm{x}\in Q_{m} : 0\leq\Phi_{m}(\bm{x}^{(i)})-x_{i}<\delta \right\},
\end{align*}
then
\[
|x_{i}-f_{m}(\bm{x}^{(i)})|-|f_{m}(\bm{x}^{(i)})-\Phi_{m}(\bm{x}^{(i)})|\leq |\Phi_{m}(\bm{x}^{(i)})-f_{m}(\bm{x}^{(i)})+f_{m}(\bm{x}^{(i)})-x_{i}|<\delta.
\]
Therefore, the above inequality and \eqref{ineqRSNNT2} imply that 
\[
|x_{i}-f_{m}(\bm{x}^{(i)})|\leq |f_{m}(\bm{x}^{(i)})-\Phi_{m}(\bm{x}^{(i)})|+\delta<2\delta,
\]
and with the argument of \eqref{eqdist} we obtain $ \bm{x}\in B_{2\delta}^{*}  $; whereby $ \widetilde{G}_{m} \subseteq B_{2\delta}^{*} $. So, by \ref{(M)cond},
\begin{equation}
	\mu(\widetilde{G}_{m})\leq \mu(B_{2\delta}^{*})\leq C_{2}(2\delta)^{\gamma}.\label{mGtilde}
\end{equation}

So far we constructed the sets $ G_{m} $ and $ \widetilde{G}_{m} $ to approximate the indicator function $ \mathbbm{1}_{\Omega} $ with the ReLU-NN $ \Phi_{m} $, only on the set $ Q_{m} $. In the following, we join these approximations and find another approximation for the same indicator function but on all $ \Omega $. To this end, we define 
	\[
	\widehat{\Phi}_{m}\left(\bm{x},\widetilde{\Phi}_{m}(\bm{x})\right):=\begin{cases}
		0 & \text{if }  \exists i: 0<b_{i}-a_{i}<2\widehat{\delta}\\
		\varrho\left( \overunderset{d}{i=1}{\sum} t_{i}(x_{i})+\varrho(\widetilde{\Phi}_{m}(\bm{x}))-d\right) & \text{otherwise,}
	\end{cases}
	\]
with $ \widehat{\delta}:=\delta^{\gamma/\alpha} $,
\begin{equation}
	t_{i}(u):=\begin{cases}
		0 & u\in[0,1]\setminus [a_{i},b_{i}]\\
		1 & u\in [a_{i}+\widehat{\delta},b_{i}-\widehat{\delta}]\\
		{(u-a_{i})}/{\widehat{\delta}} & u\in [a_{i},a_{i}+\widehat{\delta}]\\
		{(b_{i}-u)}/{\widehat{\delta}} & u\in [b_{i}-\widehat{\delta},b_{i}];
	\end{cases}\label{eqdeftid}
\end{equation}
and 
\[
\widehat{G}_{m}:=\left\{\bm{x}\in [0,1]^{d}: \widehat{\Phi}_{m}\left(\bm{x},\widetilde{\Phi}_{m}(\bm{x})\right)\neq \mathbbm{1}_{Q_{m}}\cdot \widetilde{\Phi}_{m}(\bm{x})\right\}.
\]
Note that when there exists $  i\in\{1,\ldots,d\} $ such that $ 0<b_i-a_i<2\widehat{\delta} $, we get $ Q_{m}\subset T_{\frac{1}{2}(b_i-a_i)+a_i,\widehat{\delta}}^{(i)} $, where $ T_{\frac{1}{2}(b_i-a_i)+a_i,\widehat{\delta}}^{(i)} $ is a tube as in \eqref{tube}, and therefore the tube compatibility of $ \mu $ implies that 
\begin{align}
	\mu(\widehat{G}_{m})&=\left(\left\{\bm{x}\in [0,1]^{d}: \mathbbm{1}_{Q_{m}}\cdot \widetilde{\Phi}_{m}(\bm{x})\neq 0\right\}\right)\nonumber\\
	&\leq \mu(Q_{m})\nonumber\\
	&\leq \mu\left(T_{\frac{1}{2}(b_i-a_i)+a_i,\widehat{\delta}}^{(i)}\right)\nonumber\\
	&\leq C_{3}\widehat{\delta}^{\alpha}.\label{eqbiai2del}
\end{align}
If for all $ i\in \{1,\ldots,d\} $, $ b_{i}-a_{i}\geq 2\widehat{\delta} $ is satisfied, then
\begin{itemize}
	\item If $ \bm{x}\in [0,1]^{d}\setminus Q_{m} $, exists $ x_{i} $ such that $ x_{i}\notin [a_i,b_i] $ and 
	\begin{equation}
		\sum_{i=1}^{d}t_{i}(x_{i})+\varrho(\widetilde{\Phi}_{m}(\bm{x}))-d\leq \varrho(\widetilde{\Phi}_{m}(\bm{x}))-1\leq 0,\label{eqphi0}
	\end{equation} 
	this implies that $ \widehat{\Phi}_{m}\left(\bm{x},\widetilde{\Phi}_{m}(\bm{x})\right)=0=\mathbbm{1}_{Q_{m}}\cdot \widetilde{\Phi}_{m}(\bm{x}) $.
	
	\item If $ \bm{x}\in\prod_{i=1}^{d}[a_i+\widehat{\delta},b_i-\widehat{\delta}] $, 
	\[
	\widehat{\Phi}_{m}\left(\bm{x},\widetilde{\Phi}_{m}(\bm{x})\right)=\varrho\left(\varrho(\widetilde{\Phi}_{m}(\bm{x}))\right)=\mathbbm{1}_{Q_{m}}\cdot \widetilde{\Phi}_{m}(\bm{x}).
	\]	
\end{itemize}
Then we conclude that $ \widehat{G}_{m} \subset Q_{m}\setminus\prod_{i=1}^{d}[a_{i}+\widehat{\delta},b_{i}-\widehat{\delta}] $ and analogously to \eqref{eqbiai2del} we get 
\begin{align}
	\mu(\widehat{G}_{m})&\leq \mu\left(Q_{m}\setminus\prod_{i=1}^{d}[a_{i}+\widehat{\delta},b_{i}-\widehat{\delta}]\right)\nonumber\\
	&\leq \mu\left(\bigcup_{i=1}^{d}\left[T_{a_i+\widehat{\delta}/2,\widehat{\delta}/2}^{(i)}\cup T_{b_i-\widehat{\delta}/2,\widehat{\delta}/2}^{(i)}\right]\right)\nonumber\\
	&\leq 2dC_{3}(\widehat{\delta}/2)^{\alpha}.\label{eqbiai2del2}
\end{align}
Finally, we define 
\[
G^{\prime}_{m}:=\left\{\bm{x}\in [0,1]^{d}: \mathbbm{1}_{Q_{m}\cap \Omega}(\bm{x})\neq\widehat{\Phi}_{m}\left(\bm{x},\widetilde{\Phi}_{m}(\bm{x})\right)\right\}
\]
and obtain using \eqref{mG}, \eqref{mGtilde}, \eqref{eqbiai2del} and \eqref{eqbiai2del2}, that 
\begin{align*}
	\mu\left(G^{\prime}_{m}\right)&\leq \mu(G_{m})+\mu(\widetilde{G}_{m})+ \mu(\widehat{G}_{m})\\
	&\leq C_{2}\delta^{\gamma}+C_{2}(2\delta)^{\gamma}+2dC_{3}(\widehat{\delta}/2)^{\alpha}\\&\leq \left(3C_{2}+2dC_{3}\right)\delta^{\gamma}\\
	&\leq \left(3C_{2}+2dC_{3}\right)\left(C_{1}\sqrt{(d-1)/N}\right)^{\gamma}\\&\leq 3.5d(d-1)^{\gamma/2}C_{1}^{\gamma}N^{-\gamma/2}\max\{C_{2},C_{3}\}.
\end{align*}
We put, for $ \bm{x} \in[0,1]^{d} $
\begin{equation}
	\Phi(\bm{x}):=\sum_{m=1}^{M} \widehat{\Phi}_{m}\left(\bm{x},\widetilde{\Phi}_{m}(\bm{x})\right) \label{Phioutputlatyer}
\end{equation}
and since $ \mathbbm{1}_{\Omega}=\sum_{m=1}^{M} \mathbbm{1}_{\Omega\cap Q_{m}} $ a.e., we conclude that 
\begin{align*}
	\mu\left(\left\{\bm{x}\in[0,1]^{d}:\mathbbm{1}_{\Omega}(\bm{x})\neq \Phi(\bm{x})\right\}\right)&\leq \sum_{m=1}^{M}  \mu\left(G^{\prime}_{m}\right) \\
	&\leq 3.5Md(d-1)^{\gamma/2}C_{1}^{\gamma}N^{-\gamma/2}\max\{C_{2},C_{3}\}.
\end{align*}
Moreover, for some $ m\in \{1,\ldots,M\} $, we know that $ Q_{m}\cap Q_{j}=\emptyset $ for all $ m\neq j $. If $ \bm{x}\in Q_{m} $ is fixed\footnote{If for all $ m\in \{1,\ldots,M\} $, $ \bm{x}\notin Q_{m}  $, then $ \Phi(\bm{x})=0 $ (see \eqref{eqdeftid} and \eqref{eqphi0}).}, then for all $j\neq m$, $ \widehat{\Phi}_{j}\left(\bm{x},\widetilde{\Phi}_{j}(\bm{x})\right)=0  $ (see \eqref{eqphi0}) and 
\[
0\leq \Phi(\bm{x})= \widehat{\Phi}_{m}\left(\bm{x},\widetilde{\Phi}_{m}(\bm{x})\right)\leq \varrho(\varrho(\widetilde{\Phi}_{m}(\bm{x})))\leq \widetilde{\Phi}_{m}(\bm{x})\leq 1.
\]
This completes the first part of the theorem. 

To conclude, we obtain information about the architecture, biases, and weights of $ \Phi $. 
Note that by \eqref{eqSNN} we can define
\begin{equation}
	{\Phi}_{m}(\bm{x}^{(i)})=b^{(1)}+\sum_{i=1}^{N}w_{i}^{(1)}\varrho\left(\left\langle \bm{w}_{i}^{(0)},\bm{x}^{(i)} \right\rangle+b_{i}^{(0)}\right)\label{eqneuronsl1}
\end{equation}
for some $ w_{i}^{(1)}, b_{i}^{(0)}, b^{(1)}\in\mathbb{R} $ and $ \bm{w}_{i}^{(0)}\in\mathbb{R}^{d-1} $ for $ i=1,\ldots,N $. Also, $ x_{i}=\varrho(x_{i})-\varrho(-x_{i}) $ and 
	\begin{equation}
		\widetilde{\Phi}_{m}(\bm{x})=\frac{1}{\delta}\left(\varrho\left(\Phi_{m}(\bm{x}^{(i)})-(\varrho(x_{i})-\varrho(-x_{i}) )\right)-\varrho\left(\Phi_{m}(\bm{x}^{(i)})-(\varrho(x_{i})-\varrho(-x_{i}) )-\delta\right)\right).\label{eqneuronsl2}
	\end{equation}
Finally, $ t_{i}(u)=\frac{t_{i}^{1}-t_{i}^{2}-t_{i}^{3}+t_{i}^{4}}{\widehat{\delta}}(u) $ where 
\begin{equation}
	t_{i}^{1}(u)=\varrho(u-a_{i}),~ t_{i}^{2}(u)=\varrho(u-a_{i}-\widehat{\delta}),~ t_{i}^{3}(u)=\varrho(u-b_{i}+\widehat{\delta}),~ t_{i}^{4}(u)=\varrho(u-b_{i}),\label{eqneuronsl31}
\end{equation}
and 
	\begin{equation}
		\widehat{\Phi}_{m}\left(\bm{x},\widetilde{\Phi}_{m}(\bm{x})\right):=\begin{cases}
			0 & \text{if } \exists i: 0<b_{i}-a_{i}<2\widehat{\delta}\\
			\varrho\left(\overunderset{d}{i=1}{\sum}\left(\frac{t_{i}^{1}-t_{i}^{2}-t_{i}^{3}+t_{i}^{4}}{\widehat{\delta}}(x_{i})\right)+\varrho(\widetilde{\Phi}_{m}(\bm{x}))-d\right) & \text{otherwise.}
		\end{cases}\label{eqneuronsl32}
	\end{equation}
We count the number of neurons as follows.
\begin{itemize}
	\item \textbf{Input layer.} We have $ d $ neurons by the numbers of coordinates in the vector $ \bm{x}\in[0,1]^{d} $.

	\item \textbf{Hidden layer 1.}  In \eqref{eqneuronsl1} we are using $ N $ neurons of $ \Phi_{m} $, for \eqref{eqneuronsl2} we have to find before $ \varrho(\pm x_{i}) $ and $ \varrho(\pm(\varrho(x_{i})-\varrho(-x_{i}))) $ so we need $ 2d $ and $ 2 $ more neurons.	Then we get $ 2(d+1)+N $ for each $ i=1,\ldots,d $, but we had fixed $ m $, thus in total we obtain $ M(2(d+1)+N) $.

	\item \textbf{Hidden layer 2.} For \eqref{eqneuronsl2}, we know that
		\begin{equation}
			\varrho\left(\Phi_{m}(\bm{x}^{(i)})-(\varrho(x_{i})-\varrho(-x_{i}) )\right) ~\text{ and }~\varrho\left(\Phi_{m}(\bm{x}^{(i)})-(\varrho(x_{i})-\varrho(-x_{i}) )-\delta\right) \label{cwandbl2}
		\end{equation}
	add up to two neurons, but in addition the decomposition \eqref{eqneuronsl31} of $ t_i $ for the next layer is now done using $ 4 $ neurons for each $ i $, so we have $ 4d $ more neurons. Then, in total we get $ 2M(2d+1) $.
	
	\item \textbf{Hidden layer 3.} For \eqref{eqneuronsl32}, we see that among the possibilities $ 0 $ or $ \varrho(\cdot) $ of $ \widehat{\Phi}_{m} $, we use only one neuron, however $ m\in\{1,\ldots,M\} $ is fixed, so in total we have $ M $.

	\item \textbf{Output layer.} Here we only obtain $ \Phi(\bm{x}) $ and by its definition in \eqref{Phioutputlatyer} we use one neuron for the output.
\end{itemize}
In summary, the architecture of this NN is given by
\[
\left(d,M(2(d+1)+N),2M(d+1),M,1\right)
\]
and the number of all neurons is $ M(4(d+1)+N+1)+d+1 $. Also, an upper bound for the total weights can be obtained by summing the number of neurons of each layer when multiplied by its input dimension (which is given by the output dimension of the previous layer), plus the number of non-input neurons associated with the biases and $ M $ to account for the weights of final output layer. Then
\begin{align*}
	W(\Phi)&\leq dM(2(d+1)+N)+2M(d+1)(2(d+1)+N)+2M(d+1)\\
	&\quad+M(4(d+1)+N+1)+M+1\\
	&\leq 3 d M N + 3 M N +6 d^{2} M  +16 d M +12 M +1\\
	&\leq 41 Md^{2}N.
\end{align*}

By \eqref{ineqRSNNT2}, we use $ 7\sqrt{C_{1}}  $ to bound the number of weights and biases for $ \Phi $ in absolute value.  When there exists $ i\in \{1,\ldots,d\}: 0<b_{i}-a_{i}<2\widehat{\delta} $, then we have $ \widehat{\Phi}_{m}\left(\bm{x},\widetilde{\Phi}_{m}(\bm{x})\right)=0 $ and we can choose all first layer weights of  $ \widetilde{\Phi}_{m} $  and $ \widehat{\Phi}_{m} $ as zero. Otherwise, we can bound the weights and biases of $ \widetilde{\Phi}_{m} $ and $ \widehat{\Phi}_{m} $ with $ 7\sqrt{C_{1}} $. In both cases we can bound the weights and biases for the first layer using $ 1+7\sqrt{C_{1}} $. In the second layer by \eqref{cwandbl2} we have the bound $  1+7\sqrt{C_{1}} $ for each of the neurons, and by \eqref{eqneuronsl31} we get $ 1+ \widehat{\delta}^{-1}$, since $ |t_{i}^{j}(u)|/\widehat{\delta}\leq 1+\widehat{\delta}^{-1} $ for all $ j\in \{1,2,3,4\} $. Lastly, using \eqref{eqneuronsl32} we know that for the third layer, the weights and biases are bounded by $ \max\{\delta^{-1},\widehat{\delta}^{-1},d\} $. Then, the weights and biases of $ \Phi $ are bounded in magnitude by
\begin{align*}
	&\max\left\{1+7\sqrt{C_{1}},1+ \delta^{-\gamma/\alpha},\max\{\delta^{-1},\delta^{-\gamma/\alpha},d\}\right\}\\
	&\leq\max\left\{(1+7\sqrt{C_{1}}),1+ \delta^{-\gamma/\alpha},\delta^{-1}\right\} \\
	&\leq (1+\sqrt{C_{1}})\left(7+ \delta^{-\gamma/\alpha}+\delta^{-1}\right)\\
	&\leq (1+\sqrt{C_{1}})\left(7+N/C_{1}+(N/C_{1})^{\gamma/\alpha}\right),
\end{align*}
since $ \delta^{-1}=\left(C_{1}\sqrt{(d-1)/N}\right)^{-1}=(N/(d-1))^{1/2}/C_{1}\leq N/C_{1} $.
\qed

\subsection{Proof of Theorem~\ref{secondproxtheo}}

First, we show the following result which is a consequence of \cite[Theorem A.1]{fastc}.

\begin{lem}\label{theoA1mod} Assume that $ C^{*} $ is the Bayes classifier as in \eqref{notatloss}, and:
	\begin{enumerate}[label=$ (\roman*)$]
		\item \label{item2} For a positive sequence $ \{a_{n}\}_{n\in \mathbb{N} } $, there exists a sequence of function classes $ \{\mathcal{F}_{n}\}_{n\in\mathbb{N}} $ such that 
		\[
		\mathcal{E}_{\phi}(f_{n},C^{*})\leq a_{n}
		\]
		for some $ f_{n}\in\mathcal{F}_{n}\subset \mathcal{F} $, where $\mathcal{F} $ is as in $ \eqref{defF} $.
		
		\item \label{item 4} Let $ \bm{x}\in [0,1]^{d} $ and $ y\in \{0,1\} $. Then, for all $ f_{n}\in \mathcal{F}_{n} $ and  $ n\in \mathbb{N} $,
		\[
		\mathbb{E}_{\bm{x}}\left[\left(\phi(2f_{n}(\bm{x})-1,2y-1)-\phi(2C^{*}(\bm{x})-1,2y-1)\right)^{2}\right]\leq c_{2} \mathcal{E}_{\phi}(f_{n},C^{*})
		\]
		for a constant $ c_{2}>0 $ depending only on $ \eta $ and $ \phi $, where $ \eta $ and $ \phi $ are defined in \eqref{posterandomeg} and \eqref{hinge01}.
		
		\item \label{item5} There exist $ \hat{c}_{3}>0 $ and $ \{\hat{\delta}_{n}\}_{n\in\mathbb{N}}\subset \mathbb{R}^{+} $ such that  
		\[
		V_{[0,1]^{d},\norm{\cdot}_{n}}(\hat{\delta}_{n})\leq \hat{c}_{3}n\hat{\delta}_{n},\quad \text{for all}\quad n\in\mathbb{N}.
		\]		
	\end{enumerate}
	Then with these assumptions and taking $ \epsilon^{2}:=2\max\{a_{n},2^{8}\hat{\delta}_{n}\} $, it holds for some universal constant $ c>0 $, that 	 
	\begin{equation}
		\mu\left(\mathcal{E}_{\phi}(\widehat{f}_{\phi,S},C^{*})\geq\epsilon_{n}^{2}\right)\lesssim \exp\left(-cn\epsilon_{n}^{2}\right).\label{ineqteoA1mod}
	\end{equation}
	
\end{lem}

\begin{proof}
	Note that in \cite{fastc} the labels are in the set $ \{\pm 1\} $, however, in this result we take them in $ \{0,1\} $, for this reason in \eqref{notatloss0} and \eqref{hinge01} the mapping $ 2x-1 $ from $ [0,1] $ to $ [-1,1] $ is used. Moreover, the functions $ f\in\mathcal{F} $ here are understood to have the form $ 2f-1 $ in \cite{fastc}. Now, we see that conditions from (A1) to (A5) in \cite{fastc} are satisfied as follows. \\
	
		\noindent\textbullet~ By the definition of hinge loss we know that (A1) is true since $ \phi $ is Lipschitz with constant $ c_1=1 $.\\
		
		\noindent\textbullet~  Condition (A2) turns into \ref{item2} when the labels are changed from $ \{\pm1\} $ to $ \{0,1\} $.\\
		
		\noindent\textbullet~  We know that  $ \sup_{f\in \mathcal{F}_{n}} {\norm{f}}_{\infty} \leq 1 $ with $ \mathcal{F}_{n} $ as in \ref{item2}, but also $ \sup_{f\in \mathcal{F}_{n}} {\norm{2f-1}}_{\infty} \leq 1. $  Then, in the condition (A3) we can fix $ \{F_{n}\}_{n\in \mathbb{N}} $ where $ F_{n}\gtrsim 1 $, as $ F_{n}:=1 $, for all $  n\in \mathbb{N} $.\\
		
		\noindent\textbullet~  In (A4) if we set $ v=1 $, as in the previous item we fixed $ F_n=1 $, then \ref{item 4} becomes condition (A4).\\
		
		\noindent\textbullet~  In this case, condition (A5) states the following:  There exists a sequence $ \{\delta_{n}\}_{n\in \mathbb{N}} $ such that 
		\begin{equation}\label{ineqentropy1}
			H_{B}:=H_{B}(\delta_{n}, 2\mathcal{F}_{n}-1, \norm{\cdot}_{2})\leq c_{3}n\left(\frac{\delta_{n}}{F_{n}}\right)^{2-v}=c_{3}n\delta_{n},
		\end{equation}
		for some constant $ c_{3}>0 $, with $ F_{n}=1 $ and $ v=1 $ as in the two items above. Here, to avoid unnecessary details, $ H_{B} $ is called the $ \delta_{n} $-bracketing entropy and is bounded by the inequality \cite[(A.1)]{fastc}, that is 
		\begin{equation}\label{ineqentropy2}
			H_{B}\leq \log \mathcal{N}\left(\delta_{n}/2,2\mathcal{F}_{n}-1,\norm{\cdot}_{\infty}\right):=\log\inf\mathcal{G}
		\end{equation}
		where
		\[
		\mathcal{G}:=\left\{N\in\mathbb{N}:\text{ exist } f_{1},\ldots,f_{N}\text{ such that } 2\mathcal{F}_{n}-1\subset \bigcup_{i=1}^{N}B_{\infty}(f_{i},\delta_{n})\right\}
		\]
		and $B_{\infty}(f_{i},\delta_{n}):=\left\{f\in 2\mathcal{F}_{n}-1:{\norm{f-f_{i}}}_{\infty}\leq \delta_{n}\right\}  $. Moreover, if we set $ K:=[-1,1]^{d} $ and $  G_{\delta_{n}}:=\bigcup_{i=1}^{N}B_{\infty}(f_{i},\delta_{n}) $ in \eqref{defentropy}, we know by \ref{item5} that 
		\begin{align}\label{ineqentropy3}
			\log \mathcal{N}\left(\delta_{n}/2,2\mathcal{F}_{n}-1,\norm{\cdot}_{\infty}\right)&\leq V_{[-1,1]^{d},\norm{\cdot}_{\infty}}(\delta_{n}/2)\nonumber\\
			&\leq V_{[0,1]^{d},\norm{\cdot}_{\infty}}(\delta_{n}/4)\nonumber\\ 
			&\leq \hat{c}_{3}n(\delta_{n}/4),
		\end{align}
		where $ \hat{\delta}_{n}:=\delta_{n}/4 $. Then, taking $ \hat{c}_{3}:= 4c_{3}  $ we obtain from inequalities \eqref{ineqentropy2} and \eqref{ineqentropy3} that inequality \eqref{ineqentropy1} is satisfied, that is, condition (A5) is fulfilled. 		 
		
	In conclusion, condition (A1)-(A5) are met and we fix $ c_{3}:=\left(2^{17}\max\left\{5c_{2},128\right\}\right)^{-1} $ and $ \epsilon^{2}:=2\max\{a_{n},2^{8}\hat{\delta}_{n}\} $. Then, Theorem A.1 in \cite{fastc} implies inequality \eqref{ineqteoA1mod}.	
	
\end{proof}

We now continue with the proof of Theorem~\ref{secondproxtheo}. By definition we can choose $ C^{*}:=\mathbbm{1}_{\Omega} $, therefore $ \mathcal{E}(C^{*})=0 $,
\begin{equation}\label{phiCstar0}
	\phi(2C^{*}(\bm{x})-1,2\mathbbm{1}_{\Omega}(\bm{x})-1)=\max\{0,1-(2\mathbbm{1}_{\Omega}(\bm{x})-1)^{2}\}=0
\end{equation}	
and $ \mathcal{E}_{\phi}(C^{*})=0 $. So, 
\begin{equation}
	\mathcal{E}_{\phi}(f,C^{*})=\mathcal{E}_{\phi}(f)  \quad\text{and}\quad  \mathcal{E}(f,C^{*})=\mathcal{E}(f), \quad\text{for any }\quad f\in \mathcal{F} .\label{equalsphi01risk} 
\end{equation}
Now, we show that all hypotheses of Lemma \ref{theoA1mod} are satisfied.

 Let $ \mathcal{F}_{n}:=\mathcal{NN}_{*}(d,N_{n},W_{n},B_{n}) $, then\\

	\noindent\textbullet~ By Theorem \ref{firstaproxtheo} we know that for all $ n\in\mathbb{N} $, there exists $ f_{n}\in \mathcal{F}_{n} $ such that
	\begin{small}
			\begin{align*}
			\mu\left(\left\{\bm{x}\in[0,1]^{d}:\mathbbm{1}_{\Omega}(\bm{x})\neq f_{n}(\bm{x})\right\}\right)\leq 3.5Md(d-1)^{\gamma/2}C_{1}^{\gamma}{\widehat{N}_{n}}^{-\gamma/2}\max\{C_{2},C_{3}\}\leq n^{-\gamma/(2+\gamma)}/2,
		\end{align*}
	\end{small}
	and this implies 
	\begingroup
	\allowdisplaybreaks
	\begin{align*}
		\mathcal{E}_{\phi}(f_{n})&=\mathbb{E}_{\bm{x}}\left[\phi(2f_{n}(\bm{x})-1,2\mathbbm{1}_{\Omega}-1)\right]
		\\&=\int_{[0,1]^{d}} \max\{0,1-(2f_{n}(\bm{x})-1)(2\mathbbm{1}_{\Omega}(\bm{x})-1)\} \,d\mu\\
		&=\int_{\{\mathbbm{1}_{\Omega}(\bm{x})\neq f_{n}(\bm{x})\}} \max\{0,1-(2f_{n}(\bm{x})-1)(2\mathbbm{1}_{\Omega}(\bm{x})-1)\} \,d\mu \\
		&\leq \int_{[0,1]^{d}} 2\mathbbm{1}_{\mathbbm{1}_{\Omega}(\bm{x})\neq f_{n}(\bm{x})} \,d\mu = 2\mu\left(\left\{\bm{x}\in[0,1]^{d}:\mathbbm{1}_{\Omega}(\bm{x})\neq f_{n}(\bm{x})\right\}\right)\\
		&\leq  n^{-\gamma/(2+\gamma)},
	\end{align*}
	\endgroup
	i.e, condition \ref{item2} is fulfilled with $ a_{n}:= n^{-\gamma/(2+\gamma)} $.\\

	\noindent\textbullet~ Let $ G=\left\{\bm{x}\in [0,1]^{d}:\phi(2f_{n}(\bm{x})-1,2\mathbbm{1}_{\Omega}(\bm{x})-1)\leq 1\right\} $. Using identity \eqref{phiCstar0}, we have 
	\begin{align}
		&\mathbb{E}_{\bm{x}}\left[\left(\phi(2f_{n}(\bm{x})-1,2\mathbbm{1}_{\Omega}(\bm{x})-1)-\phi(2C^{*}(\bm{x})-1,2\mathbbm{1}_{\Omega}(\bm{x})-1)\right)^{2}\right]\nonumber\\
		&=\mathbb{E}_{\bm{x}}\left[\left(\phi(2f_{n}(\bm{x})-1,2\mathbbm{1}_{\Omega}(\bm{x})-1)\right)^{2}\right]\\&=\int_{\mathbb{R}^{d}} \left(\phi(2f_{n}(\bm{x})-1,2\mathbbm{1}_{\Omega}(\bm{x})-1)\right)^{2} d\mu\nonumber\\
		&\leq\int_{G} \phi(2f_{n}(\bm{x})-1,2\mathbbm{1}_{\Omega}(\bm{x})-1) d\mu+\int_{G^{c}} 4 d\mu\nonumber\\ &\leq \mathbb{E}_{\bm{x}}\left[\phi(2f_{n}(\bm{x})-1,2\mathbbm{1}_{\Omega}(\bm{x})-1)\right]+ 4\mu(G^{c})\nonumber\\
		&\leq \mathcal{E}_{\phi}(f_{n})+ 4\mathbb{E}_{\bm{x}}\left[C_{f_{n}}(\bm{x})\neq y\right]\nonumber\\
		&\leq \mathcal{E}_{\phi}(f_{n})+4\mathcal{E}(f_{n})\label{ineqforitem4}
	\end{align}
	since $ |\phi(2f_{n}(\bm{x})-1,2\mathbbm{1}_{\Omega}(\bm{x})-1)|\leq |1-(2f_{n}(\bm{x})-1)(2\mathbbm{1}_{\Omega}(\bm{x})-1)|\leq 2 $ and $ x\in G^{c} $ implies 
	\[
	\phi(2f_{n}(\bm{x})-1,2\mathbbm{1}_{\Omega}(\bm{x})-1)>1,\quad\text{ that is }\quad (2f_{n}(\bm{x})-1)(2\mathbbm{1}_{\Omega}(\bm{x})-1)<0,
	\]
	i.e. $ C_{f_{n}}(\bm{x})\neq y $. Then, by inequalities \eqref{ineqforhinge} and \eqref{ineqforitem4},
	\[
	\mathbb{E}_{\bm{x}}\left[\left(\phi(2f_{n}(\bm{x})-1,2\mathbbm{1}_{\Omega}(\bm{x}) -1)\right)^{2}\right]\leq \mathcal{E}_{\phi}(f_{n})+4\mathcal{E}(f_{n})\leq (1+4C_{\phi})\mathcal{E}_{\phi}(f_{n})
	\]
	and condition \ref{item 4} with $ c_{2}:=1+4C_{\phi} $ is satisfied.\\
	
	\noindent\textbullet~ Using Lemma \ref{lemaentropy} with $ \delta=\hat{\delta}_{n}:=a_n (1+\log n)$ we obtain
	\begin{align*}
		V_{[0,1]^{d},\norm{\cdot}_{\infty}}(\hat{\delta}_{n})&\leq W_{n}\cdot\left(10+\log(1/\hat{\delta}_{n})+5\log(\lceil B_{n} \rceil)+5\log(\max\{d,W_{n}\})\right)\\
		&\lesssim  \widehat{N}_{n} (1+\log n)\\
		& \lesssim  n\cdot n^{2/(2+\gamma)-1}(1+\log n)\\
		&=n\hat{\delta}_{n},\quad \text{for all}\quad n\in\mathbb{N}.
	\end{align*}
	So, there exists $ \hat{c}_{3}>0 $ such that
	\[
	V_{[0,1]^{d},\norm{\cdot}_{n}}(\hat{\delta}_{n})\leq \hat{c}_{3}n\hat{\delta}_{n},\quad \text{for all}\quad n\in\mathbb{N}
	\]
	and condition \ref{item5}  is fulfilled.

Then, we use  Lemma \ref{theoA1mod} with
\[
\epsilon_{n}^{2}=2\max\{a_{n},2^{8}a_{n}(1+\log n)\}=2^{9}n^{-\gamma/(2+\gamma)}(1+\log n)
\]
and we get
\begin{align}
	\mu\left(\mathcal{E}_{\phi}(\widehat{f}_{\phi,S})\gtrsim n^{-\gamma/(2+\gamma)}(1+\log n)  \right)&\leq \mu\left(\mathcal{E}_{\phi}(\widehat{f}_{\phi,S})\geq \epsilon_{n}^{2}  \right)\nonumber\\&\lesssim \exp\left(-n^{1-\gamma/(2+\gamma)}(1+\log n)\right)\nonumber\\
	&\lesssim n^{-1}.\label{ineqforfinalteoaprox2}
\end{align}
Finally, by inequalities \eqref{ineqforhinge}, \eqref{equalsphi01risk} and \eqref{ineqforfinalteoaprox2} we conclude \eqref{conclusiontaprox2}.

\qed

\end{document}